\documentclass{article}

\usepackage[abbrev]{amsrefs} 

\usepackage[preprint]{neurips_2021}

\usepackage[utf8]{inputenc} 
\usepackage[T1]{fontenc}    
\usepackage{hyperref}       
\usepackage{url}            
\usepackage{booktabs}       
\usepackage{amsfonts}       
\usepackage{nicefrac}       
\usepackage{microtype}      
\usepackage{xcolor}         

\usepackage{amsthm}
\usepackage{amssymb}
\usepackage{amstext}
\usepackage{amsmath}
\usepackage{amscd}
\usepackage{amsfonts}
\usepackage{enumerate}
\usepackage[pdftex]{graphicx}

\usepackage{booktabs}
\usepackage{latexsym}
\usepackage{mathrsfs}
\usepackage{mathtools}
\usepackage[subrefformat=parens]{subcaption}
\usepackage{verbatim}
\usepackage{bm}
\usepackage{tikz}
\usepackage{authblk}

\usepackage{algorithm}
\usepackage{algorithmic}
\usepackage{natbib}

\newtheorem{thm}{Theorem}[section]

\newtheorem{lem}[thm]{Lemma}
\newtheorem{prop}[thm]{Proposition}
\newtheorem{ass}[thm]{Assumption}
\theoremstyle{definition}
\newtheorem{rem}[thm]{Remark}

\newtheorem*{prop*}{Proposition}
\theoremstyle{remark}

\numberwithin{equation}{section}

\title{Spectral Pruning for Recurrent Neural Networks}

\author[1,a]{\rm Takashi Furuya}
\author[2,b]{\rm Kazuma Suetake}
\author[3,c]{\rm Koichi Taniguchi}
\author[4,d]{\rm Hiroyuki Kusumoto}
\author[2,e]{\rm Ryuji Saiin}
\author[2,f]{\rm Tomohiro Daimon}
\affil[1]{{\small Department of Mathematics, Hokkaido University, Japan}}
\affil[a]{{\small Email: takashi.furuya0101@gmail.com}\vspace{3mm}}
\affil[2]{{\small AISIN SOFTWARE Co., Ltd., Japan}}
\affil[b]{{\small Email: kazuma.suetake@aisin-software.com}\vspace{3mm}}

\affil[3]{{\small Advanced Institute for Materials Research, 
Tohoku University, Japan}}

\affil[c]{{\small Email: koichi.taniguchi.b7@tohoku.ac.jp}\vspace{3mm}}

\affil[4]{{\small Graduate School of Mathematics, Nagoya University, Japan}}

\affil[d]{{\small Email: mackey3141@gmail.com}\vspace{3mm}}

\affil[e]{{\small Email: ryuji.saiin@aisin-software.com}\vspace{3mm}}

\affil[f]{{\small Email: tomohiro.daimon@aisin-software.com}\vspace{3mm}}

\begin{document}

\maketitle

\begin{abstract}
Recurrent neural networks (RNNs) are a class of neural networks used in sequential tasks. 
However, in general, RNNs have a large number of parameters and involve enormous computational costs by repeating the recurrent structures in many time steps. 
As a method to overcome this difficulty, RNN pruning has attracted increasing attention in recent years, and it brings us benefits in terms of the reduction of computational cost as the time step progresses.
However, most existing methods of RNN pruning are heuristic. 
The purpose of this paper is to study the theoretical scheme for RNN pruning method. 
We propose an appropriate pruning algorithm for RNNs inspired by “spectral pruning”, and provide the generalization error bounds for compressed RNNs. 
We also provide numerical experiments to demonstrate our theoretical results and show the effectiveness of our pruning method compared with existing methods.
\end{abstract}

\section{Introduction}\label{Introduction}
Recurrent neural networks (RNNs) are a class of neural networks used in sequential tasks. However, in general, RNNs have a large number of parameters and involve enormous computational costs by repeating the recurrent structures in many time steps. These make their application difficult in edge-computing devices. 
To overcome this difficulty, RNN compression has attracted increasing attention in recent years. 
It brings us more benefits in terms of the reduction of computational costs as the time step progresses, compared to deep neural networks (DNNs) without any recurrent structure. 
There are many RNN compression methods such as pruning [\cites{Narang, tang2015pruning, Zhang, Lobacheva, wang2019acceleration, Wen2020StructuredPO, lobacheva2020structured}, low rank factorization [\cites{Kliegl, Tjandra}, quantization [\cites{Alom, Liu}, distillation [\cites{Shi, Tang}, and sparse training [\cites{liu2021selfish, liu2021efficient, dodge2019rnn, wen2017learning}. 
This paper is devoted to the pruning for RNNs, and its purpose is to provide an RNN pruning method with the theoretical background. 
\par
Recently, Suzuki et al. [\cites{Suzuki} proposed a novel pruning method with the theoretical background, called {\it spectral pruning}, for DNNs such as the fully connected and convolutional neural network architectures. The idea of the proposed method is to select important nodes for each layer by minimizing the information losses (see (2) in [\cites{Suzuki}), which can be represented by the layerwise covariance matrix. The minimization only requires linear algebraic operations. Suzuki et al. [\cites{Suzuki} also evaluated generalization error bounds for networks compressed using spectral pruning (see Theorems 1 and 2 in [\cites{Suzuki}). It was shown that generalization error bounds are controlled by the {\it degrees of freedom}, which are defined based on the eigenvalues of the covariance matrix. Hence, the characteristics of the eigenvalue distribution have an influence on the error bounds. We can also observe that in the generalization error bounds, there is a bias-variance tradeoff corresponding to compressibility. Numerical experiments have also demonstrated the effectiveness of spectral pruning.
\par
In this paper, we extend the theoretical scheme of spectral pruning to RNNs. Our pruning algorithm involves the selection of hidden nodes by minimizing the information losses, which can be represented by the time mean of covariance matrices instead of the layerwise covariance matrix which appears in spectral pruning of DNNs. We emphasize that our information losses are derived from the generalization error bound. More precisely, we show that choosing compressed weight matrices which minimize the information losses reduces the generalization error bound we evaluated in Section \ref{The error bound for the general compressed RNN} (see sentences after Theorem \ref{generalization error bound}). We also remark that Suzuki et al. [\cites{Suzuki} has not clearly mentioned anything about how the information losses are derived from. As in DNNs [\cites{Suzuki}, we can provide the generalization error bounds for RNNs compressed with our pruning and interpret the degrees of freedom and the bias-variance tradeoff.
\par
We also provide numerical experiments to compare our method with existing methods. We observed that our method outperforms existing methods,
and gets benefits from over-parameterization [\cites{chang2020provable, zhang2021understanding} (see Sections \ref{Pixel-MNIST} and \ref{Language Modeling}). In particular, our method can compress models with small degradation (see Remark \ref{regularization parameter is zero}) when we employ IRNN, which is an RNN that uses the ReLU as the activation function and initializes weights as the identity matrix and biases to zero (see [\cites{Le}).
\par
The summary of our contributions is the following:
\begin{itemize}
    \item A pruning algorithm for RNNs (Section \ref{Pruning algorithm}) is proposed by the analysis of generalization error (Remark \ref{approx is bounded by infor} and Theorem \ref{generalization error bound for SP}).
    
    \item The generalization error bounds for RNNs compressed with our pruning algorithm are provided (Theorem \ref{generalization error bound for SP}).
\end{itemize}
\section{Related Works}\label{Related Works}
One of the popular compression methods for RNNs is pruning that removes redundant weights based on certain criteria.  For example, magnitude-based weight pruning [\cites{Narang, narang2017block, tang2015pruning} involves pruning trained weights that are less than the threshold value decided by the user. This method has to gradually repeat pruning and retraining weights to ensure that a certain accuracy is maintained. 
However, based on recent developments, the costly repetitions might not always be necessary. In one-shot pruning [\cites{Zhang, lee2018snip}, weights are pruned once prior to training from the spectrum of the recurrent Jacobian. Bayesian sparsification [\cites{Lobacheva, molchanov2017variational} induce sparse weight matrix by choosing the prior as log-uniform distribution, and weights are also once pruned if the variance of the posterior over weight is large. 
\par 
While the above methods are referred to as weight pruning, our spectral pruning is a structured pruning where redundant nodes are removed. The advantage of the structured pruning over the weight pruning is that it more simply reduces computational costs. The implementation advantages of structured pruning are illustrated in [\cites{wang2019acceleration}.
Although weight pruning from large networks to small networks is less likely to degrade accuracy, it usually requires an accelerator for addressing sparsification (see [\cites{Parashar}). 
The structured pruning methods discussed in [\cites{wang2019acceleration, Wen2020StructuredPO, lobacheva2020structured} induce sparse weight matrices in the training process, and prune weights close to zero, and does not repeat fine-tuning. 
In our pruning, weight matrices are trained by the usual way, and compressed weight matrices consist of the multiplication of the trained weight matrix and the reconstruction matrix, and no need to repeat  pruning and fine-tuning. The idea of the multiplication of the trained weight matrix and the reconstruction matrix is a similar idea to low rank factorization [\cites{Kliegl, Tjandra, prabhavalkar2016compression, grachev2019compression, denil2013predicting}. 
In particular, the work [\cites{denil2013predicting} is most related to spectral pruning, and it employs the reconstruction matrix replacing the empirical covariance matrix with kernel matrix (see Section 3.1 in [\cites{denil2013predicting}).
\par
In general, RNN pruning is more difficult than DNN pruning, because recurrent architectures are not robust to pruning, that is, even a little pruning causes accumulated errors and  total errors increase significantly for many time steps. Such a peculiar problem for recurrent feature is also observed in dropout (see Introduction in [\cites{Gal, Zaremba}). 
\par
Our motivation is to theoretically propose the RNN pruning algorithm. Inspired by [\cites{Suzuki}, we focus on the generalization error bound, and we provide the algorithm so that the generalization error bound becomes smaller. Thus, the derivation of our pruning method would be theoretical, while that of existing methods such as the magnitude-based pruning [\cites{Narang, narang2017block, tang2015pruning, wang2019acceleration, Wen2020StructuredPO} would be heuristic. For the study of the generalization error bounds for RNNs, we refer to [\cites{tu2019understanding, Chen, akpinar2019sample, joukovsky2021generalization}.
\section{Pruning Algorithm}\label{Pruning algorithm}
We propose a pruning algorithm for RNNs inspired by [\cites{Suzuki}. See Appendix~\ref{Review of the spectral pruning} for a review of spectral pruning for DNNs. Let $D=\{(X^{i}_{T}, Y^{i}_{T})\}_{i=1}^{n}$ be the training data with time series sequences  $X^{i}_{T}=(x^{i}_{t})_{t=1}^{T}$ and $Y^{i}_{T}=(y^{i}_{t})_{t=1}^{T}$, where $x^{i}_{t} \in \mathbb{R}^{d_x}$ is an input and $y^{i}_{t} \in \mathbb{R}^{d_y}$ is an output at time $t$. The training data are independently identically distributed. To train the appropriate relationship between input $X_T=(x_{t})_{t=1}^{T}$
and output $Y_T=(y_{t})_{t=1}^{T}$, we consider RNNs $f=(f_t)_{t=1}^{T}$ as
\[
f_{t}=W^{o}h_{t}+b^{o}, \quad h_{t}=\sigma(W^{h}h_{t-1}+W^{i}x_{t}+b^{hi}),
\]
for $t=1,\ldots,T$, where $\sigma:\mathbb{R} \to \mathbb{R}$ is an activation function, $h_{t} \in \mathbb{R}^{m}$ is the hidden state with the initial state $h_{0}=0$, $W^{o} \in \mathbb{R}^{d_y \times m}$, $W^{h} \in \mathbb{R}^{m \times m}$, and $W^{i} \in \mathbb{R}^{m \times d_{x}}$ are weight matrices, and $b^{o} \in \mathbb{R}^{d_y}$ and $b^{hi} \in \mathbb{R}^{m}$ are biases. Here, an element-wise activation operator is employed, i.e., we define $\sigma(x) := (\sigma(x_1),\ldots, \sigma(x_m))^{T}$ 
for $x = (x_1,\ldots, x_m) \in \mathbb{R}^{m}$.
\par
Let $\widehat{f}=(\widehat{f}_{t})_{t=1}^{T}$ be a trained RNN obtained from the training data $D$ with weight matrices $\widehat{W}^{o} \in \mathbb{R}^{d_y \times m}$, $\widehat{W}^{h} \in \mathbb{R}^{m \times m}$, and $\widehat{W}^{i} \in \mathbb{R}^{m \times d_x}$, and biases $\widehat{b}^{o} \in  \mathbb{R}^{d_y}$ and $\widehat{b}^{hi} \in  \mathbb{R}^{m}$, i.e., $\widehat{f}_{t}=\widehat{W}^{o}\widehat{h}_{t}+\widehat{b}^{o}$,  $\widehat{h}_{t}=\sigma(\widehat{W}^{h}\widehat{h}_{t-1}+\widehat{W}^{i}x_{t}+\widehat{b}^{hi})$ for $t=1,\ldots,T$. We denote the hidden state $\widehat{h}_{t}$ by
\[
\widehat{h}_{t} = \phi(x_t, \widehat{h}_{t-1}),
\]
as a function with inputs $x_t$ and $\widehat{h}_{t-1}$. Our aim is to compress the trained network $\widehat{f}$ to the smaller network $f^{\sharp}$ without loss of performance to the extent possible. 
\par
Let $J \subset [m]$ be an index set with $|J|=m^{\sharp}$, where $[m]:=\{1,\ldots,m \}$, and let $m^{\sharp} \in \mathbb{N}$ be the number of hidden nodes for a compressed RNN $f^{\sharp}$ with $m^{\sharp} \leq m$. 
We denote by $\phi_{J}(x_t, \widehat{h}_{t-1})=(\phi_{j}(x_t, \widehat{h}_{t-1}))_{j \in J}$ the subvector of $\phi(x_t, h_{t-1})$ corresponding to the index set $J$, where $\phi_{j}(x_t, \widehat{h}_{t-1})$ represents the $j$-th components of the vector $\phi(x_t, \widehat{h}_{t-1})$.
\par
{\bf (i) Input information loss.} 
The input information loss is defined by
\begin{equation}
L^{(A)}_{\tau}(J):= \min_{A \in \mathbb{R}^{m \times m^{\sharp} } } \big\{ \| \phi - A \phi_{J} \|^{2}_{n,T} + \| A \|^{2}_{\tau} \big\}, \label{input information loss}
\end{equation}
where $\| \cdot \|_{n,T}$ is the empirical $L^2$-norm with respect to $n$ and $t$, i.e., 
\[
\begin{split}
& \| \phi - A\phi_{J} \|^{2}_{n,T} \\
&
:=\frac{1}{nT} \sum_{i=1}^{n}\sum_{t=1}^{T} \big\| \phi(x^{i}_{t}, \widehat{h}^{i}_{t-1}) - A \phi_{J}(x^{i}_{t}, \widehat{h}^{i}_{t-1}) \big\|^{2}_{2},
\end{split}
\]
where $\| \cdot \|_{2}$ is the Euclidean norm, $\| A \|^{2}_{\tau}:=\mathrm{Tr}[AI_{\tau}A^{T}]$ for the regularization parameter $\tau \in \mathbb{R}^{m^{\sharp}}_{+}:=\{x \in \mathbb{R}^{m^{\sharp}_{l}} \,|\, x_{j}>0, \ j=1,\ldots,m^{\sharp}_{l} \}$, and $I_{\tau}:=\text{diag}(\tau)$. Here, $\widehat{\Sigma}_{I,I'} \in \mathbb{R}^{K \times H}$ denotes the submatrix of $\widehat{\Sigma}$ corresponding to the index sets $I, I' \subset [m]$ with $|I|=K$, $|I'|=H$, i.e., $\widehat{\Sigma}_{I,I'} = (\widehat{\Sigma}_{i,i'})_{i\in I,i'\in I'}$. Based on the linear regularization theory (see e.g., [\cites{gockenbach2016linear}), there exists a unique solution $\widehat{A}_{J} \in \mathbb{R}^{m \times m^{\sharp}}$ of the minimization problem of $\| \phi - A \phi_{J} \|^{2}_{n,T} + \| A \|^{2}_{\tau}$, which has the form
\begin{equation}
\widehat{A}_{J}=\widehat{\Sigma}_{[m],J}\big(\widehat{\Sigma}_{J,J}+I_{\tau} \big)^{-1}, \label{A_J}
\end{equation}
where $\widehat{\Sigma}$ is the (noncentered) empirical covariance matrix of the hidden state $\phi(x_t, \widehat{h}_{t-1})$ with respect to $n$ and $t$, i.e., 
\begin{equation}\label{mean-covariance}
\widehat{\Sigma}=\frac{1}{nT}\sum_{i=1}^{n}\sum_{t=1}^{T}\phi(x^{i}_t, \widehat{h}^{i}_{t-1})\phi(x^{i}_t, \widehat{h}^{i}_{t-1})^{T}.
\end{equation}
We term the unique solution $\widehat{A}_{J}$ as the {\it reconstruction matrix}. Here, we would like to emphasize that the mean of the covariance matrix with respect to time $t$ is employed in RNNs, while the layerwise covariance matrix is employed in DNNs (see Appendix \ref{Review of the spectral pruning}). By substituting the
explicit formula of the reconstruction matrix $\widehat{A}_{J}$ into (\ref{input information loss}), the input information loss is reformulated as:
\begin{equation}
L^{(A)}_{\tau}(J)=\mathrm{Tr}\Big[ \widehat{\Sigma} - \widehat{\Sigma}_{[m],J}\big( \widehat{\Sigma}_{J,J}+I_{\tau} \big)^{-1}\widehat{\Sigma}_{J,[m]} \Big]. \label{reformulate input infor}
\end{equation}
\par
{\bf (ii) Output information loss.} The hidden state of a RNN is forwardly propagated to the next hidden state or output, and hence, the two output information losses are defined by
\begin{equation}
L^{(B,o)}_{\tau}(J)
:=\sum_{j=1}^{d_y}\min_{\beta \in \mathbb{R}^{m^{\sharp}}} \Big\{ \big\| \widehat{W}^{o}_{j,:} \phi - \beta^{T} \phi_{J} \big\|^{2}_{n,T} + \big\| \beta^{T} \big\|^{2}_{\tau} \Big\}, \label{output information loss-o}
\end{equation}
\begin{equation}
L^{(B,h)}_{\tau}(J):= \sum_{j \in J} \min_{\beta \in \mathbb{R}^{m^{\sharp}}} \Big\{ \big\| \widehat{W}^{h}_{j,:} \phi - \beta^{T} \phi_{J} \big\|^{2}_{n,T} + \big\| \beta^{T} \big\|^{2}_{\tau} \Big\}, \label{output information loss-h}
\end{equation}
where $\widehat{W}^{o}_{j,:}$ and $\widehat{W}^{h}_{j,:}$ denote the $j$-th rows of the matrix $\widehat{W}^{h}$ and $\widehat{W}^{o}$, respectively. Then, the unique solutions of the minimization problems of $ \| \widehat{W}^{o}_{j,:} \phi - \beta^{T} \phi_{J} \|^{2}_{n,T}$ $+$ $\| \beta^{T} \|^{2}_{\tau}$ and $\| \widehat{W}^{h}_{j,:} \phi - \beta^{T} \phi_{J} \|^{2}_{n,T}$ $+$ $\| \beta^{T} \|^{2}_{\tau}$ are $\widehat{\beta}^{o}=(\widehat{W}^{o}_{j,:}\widehat{A}_{J})^{T}$ and $\widehat{\beta}^{h}_{j}=(\widehat{W}^{h}_{j,:}\widehat{A}_{J})^{T}$, respectively. By substituting them into (\ref{output information loss-o}) and (\ref{output information loss-h}), the output information losses are reformulated as
\begin{equation}
\begin{split}
&L^{(B,o)}_{\tau}(J) \\
&
 = \mathrm{Tr}\bigg[ \widehat{W}^{o} \Big(\widehat{\Sigma} - \widehat{\Sigma}_{[m],J}\big( \widehat{\Sigma}_{J,J}+I_{\tau} \big)^{-1}\widehat{\Sigma}_{J,[m]} \Big)\widehat{W}^{o^{T}} \bigg], \label{reformulate o output infor}
\end{split}
\end{equation}
\begin{equation}
\begin{split}
&L^{(B,h)}_{\tau}(J) \\
&
= \mathrm{Tr}\bigg[ \widehat{W}^{h}_{J, [m]} \Big(\widehat{\Sigma} - \widehat{\Sigma}_{[m],J}\big( \widehat{\Sigma}_{J,J}+I_{\tau} \big)^{-1}\widehat{\Sigma}_{J,[m]} \Big)\widehat{W}^{h^{T}}_{J, [m]} \bigg]. \label{reformulate h output infor}
\end{split}
\end{equation}
Here, we remark that the output information losses $L^{(B,o)}_{\tau}(J)$ and $L^{(B,h)}_{\tau}(J)$ are bounded above by the input information loss $L^{(A)}_{\tau}(J)$ (see Remark \ref{approx is bounded by infor}).
\par
{\bf (iii) Compressed RNNs.} We construct the compressed RNN $f^{\sharp}_{J}$ by $f^{\sharp}_{J,t}=W^{\sharp o}_{J}h^{\sharp}_{J,t}+b^{\sharp o}_{J}$ and $h^{\sharp}_{J,t}=\sigma(W^{\sharp h}_{J}h^{\sharp}_{J,t-1}+W^{\sharp i}_{J}x_{t}+b^{\sharp hi}_{J})$ for $t=1,\ldots,T$, where $W^{\sharp o}_{J}:=\widehat{W}^{o}\widehat{A}_{J}$, $W^{\sharp h}_{J}:=\widehat{W}^{h}_{J,[m]}\widehat{A}_{J}$, $W^{\sharp i}_{J}:=\widehat{W}^{i}_{J,[d_x]}$, $b^{\sharp hi}_{J}:=\widehat{b}^{hi}_{J}$, and $b^{\sharp o}_{J}:=\widehat{b}^{o}$.
\par
{\bf (iv) Optimization.} 
To select an appropriate index set $J$, we consider the following optimization problem that minimizes the convex combination of the input and two output information losses:
\begin{equation}
\min_{\scriptsize{\begin{array}{c}J \subset [m] \\ s.t.\ |J|=m^{\sharp} \end{array}}} \left\{\theta_1L^{(A)}_{\tau}(J) + \theta_2 L^{(B,o)}_{\tau}(J)+ \theta_3L^{(B,h)}_{\tau}(J) \right\}, \label{optimization-RNN}
\end{equation}
for $\theta_1, \theta_2, \theta_3 \in [0,1]$ with $\theta_1 + \theta_2 + \theta_3=1$, 
where $m^{\sharp}_{l} \in [m]$ is a prespecified number.
The optimal index $J^{\sharp}$ is obtained by the greedy algorithm. 
We term this method as {\it spectral pruning} (for a schematic diagram of spectral pruning, see Figure \ref{schematic diagram}).
The reason why information losses are employed in the objective will be theoretically explained later, when the error bounds in Remark \ref{approx is bounded by infor} and Theorem \ref{generalization error bound} are provided. We summarize our pruning algorithm in the following.
\par
\begin{algorithm}[t]             
\caption{Spectral pruning}         
\label{alg1}                          
\begin{algorithmic}[1]       
\REQUIRE  Data set $D=\{(x_{n}, y_{n})\}_{n=1}^{N} \subset \mathbb{R}^{d_x}\times \mathbb{R}^{d_y}$, Trained RNN $\widehat{f}=(\widehat{f}_{t})_{t=1}^{T}$ with $\widehat{f}_{t}=\widehat{W}^{o}\widehat{h}_{t}+\widehat{b}^{o}$,  $\widehat{h}_{t}=\sigma(\widehat{W}^{h}\widehat{h}_{t-1}+\widehat{W}^{i}x_{t}+\widehat{b}^{hi})$, Number of hidden nodes $m$ for trained RNN $\widehat{f}$, Number of hidden nodes $m^{\sharp} \leq m$ for returned compressed RNN, Regularization parameter $\tau \in \mathbb{R}^{m^{\sharp}}_{+}$, Coefficients $\theta_1, \theta_2, \theta_3 \in [0,1]$ with $\theta_1 + \theta_2 + \theta_3=1$.
\vspace{2mm}
\STATE Minimize $ \left\{\theta_1L^{(A)}_{\tau}(J) + \theta_2 L^{(B,o)}_{\tau}(J)+ \theta_3L^{(B,h)}_{\tau}(J) \right\}$ for index $J \subset [m]$ with $|J|=m^{\sharp}$ by the greedy algorithm where $L^{(A)}_{\tau}(J)$, $L^{(B,o)}_{\tau}(J)$, and $L^{(B,h)}_{\tau}(J)$ compute (\ref{reformulate input infor}), (\ref{reformulate o output infor}), and (\ref{reformulate h output infor}), respectively. 
\vspace{1mm}
\STATE Obtain optimal $J^{\sharp}$.
\vspace{1mm}
\STATE Compute $\widehat{A}_{J^{\sharp}}$ by (\ref{A_J}).
\vspace{1mm}
\STATE Set $W^{\sharp o}_{J^{\sharp}}:=\widehat{W}^{o}\widehat{A}_{J^{\sharp}}$, $W^{\sharp h}_{J^{\sharp}}:=\widehat{W}^{h}_{J^{\sharp},[m]}\widehat{A}_{J^{\sharp}}$, $W^{\sharp i}_{J^{\sharp}}:=\widehat{W}^{i}_{J^{\sharp},[d_x]}$, $b^{\sharp hi}_{J^{\sharp}}:=\widehat{b}^{hi}_{J^{\sharp}}$, $b^{\sharp o}_{J^{\sharp}}:=\widehat{b}^{o}$.
\vspace{1mm}
\RETURN Compressed RNN $f^{\sharp}_{J^{\sharp}}=(f^{\sharp}_{J^{\sharp}, t})_{t=1}^{T}$ with $f^{\sharp}_{J^{\sharp},t}=W^{\sharp o}_{J^{\sharp}}h^{\sharp}_{J^{\sharp},t}+b^{\sharp o}_{J^{\sharp}}$ and $h^{\sharp}_{J^{\sharp},t}=\sigma(W^{\sharp h}_{J^{\sharp}}h^{\sharp}_{J^{\sharp},t-1}+W^{\sharp i}_{J^{\sharp}}x_{t}+b^{\sharp hi}_{J^{\sharp}})$.
\end{algorithmic}
\end{algorithm}
\begin{figure}[h]
\begin{tabular}{c}
\begin{minipage}{0.45\hsize}
  \begin{center}
   \includegraphics[scale=0.45]{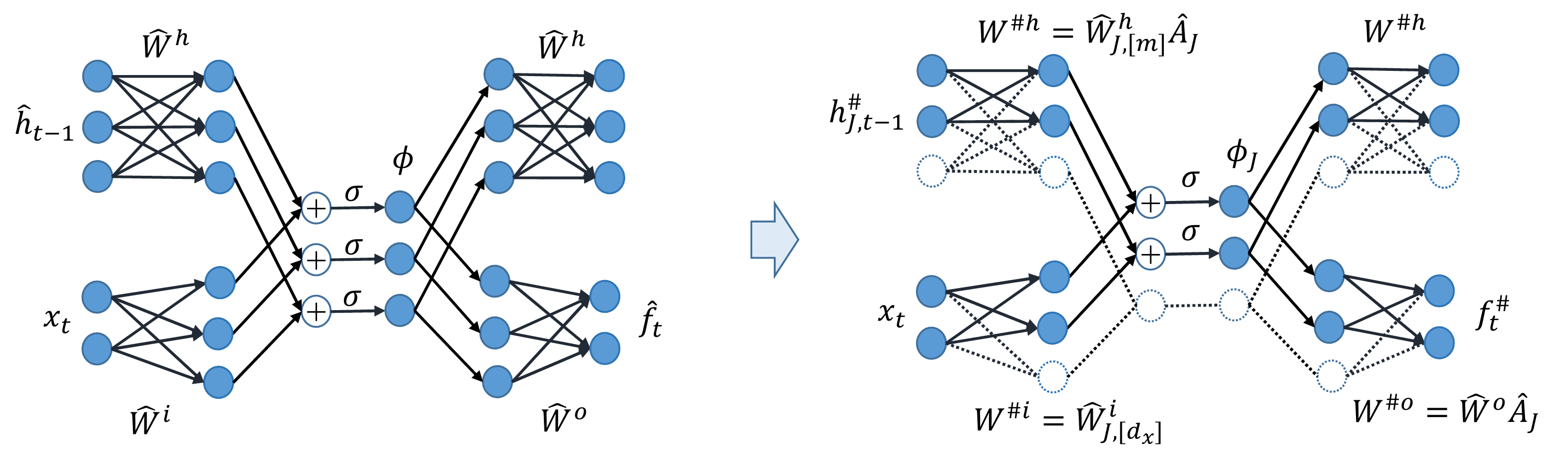}
  \end{center}
 \end{minipage}
\end{tabular}
\caption{Spectral pruning for RNN}
\label{schematic diagram}
\end{figure}
\begin{rem}
In the case of the regularization parameter $\tau=0$, spectral pruning can be applied, but the following point must be noted. In this case, the uniqueness of the minimization problem of $\left\| \phi - A \phi_{J} \right\|^{2}_{n,T}$ with respect to $A$ does not generally hold (i.e., there might be several reconstruction matrices). One of the solutions is $\widehat{A}_{J}=\widehat{\Sigma}_{[m],J}\widehat{\Sigma}_{J,J}^{\dag}$, which is the limit of (\ref{A_J}) as $\tau \to 0$, where $\widehat{\Sigma}_{J,J}^{\dag}$ is the pseudo-inverse of $\widehat{\Sigma}_{J,J}$. It should be noted that $\widehat{\Sigma}_{J,J}^{\dag}$ coincides with the usual inverse $\widehat{\Sigma}_{J,J}^{-1}$, when
$m^{\sharp}$ is smaller than or equal to the rank of the covariance matrix $\widehat{\Sigma}$.
\end{rem}
\begin{rem}\label{regularization parameter is zero}
We consider the case of the regularization parameter $\tau=0$ and $m^{\sharp} \ge m_{\mathrm{nzr}}$, where $m_{\mathrm{nzr}}$ denotes the number of non-zero rows of $\widehat{\Sigma}$. Here, we would like to remark on the relation between 
$m_{\mathrm{nzr}}$ and pruning. Let $J_{\mathrm{nzr}}$ be the index set such that $[m]\setminus J_{\mathrm{nzr}}$ corresponds to zero rows of $\widehat{\Sigma}$. Then, by the definition \eqref{mean-covariance} of $\widehat{\Sigma}$, we have for $i=1,\cdots,n$, $t=1,\cdots,T$, $v \in [m]\setminus J_{\mathrm{nzr}}$
\begin{equation*}
\phi_{v}(x_{t}^{i}, \widehat{h}^{i}_{t-1})=0, \label{non-zero rows}
\end{equation*}
which implies that $\widetilde{A}_{J_{\mathrm{nzr}}}=I_{[m],J_{\mathrm{nzr}}}$ is a trivial  solution of the minimization problem because $\| \phi - \widetilde{A}_{J_{\mathrm{nzr}}} \phi_{J_{\mathrm{nzr}}} \|^{2}_{n,T}=0$. Here, $I_{[m],J_{\mathrm{nzr}}}$ is 
the submatrix of the identity matrix corresponding to the index sets $[m]$ and $J_{\mathrm{nzr}}$.
If we choose $\widetilde{A}_{J_{\mathrm{nzr}}}=I_{[m],J_{\mathrm{nzr}}}$ as the reconstruction matrix, then the trivial compressed weights can be obtained by simply removing the columns corresponding to $[m]\setminus J_{\mathrm{nzr}}$, i.e., $W^{\sharp o}:=\widehat{W}^{o}\widetilde{A}_{J_{\mathrm{nzr}}}=\widehat{W}^{o}_{[m],J_{\mathrm{nzr}}}$ and $W^{\sharp h}:=\widehat{W}^{h}_{J_{\mathrm{nzr}},[m]}\widetilde{A}_{J_{\mathrm{nzr}}}=\widehat{W}^{h}_{J_{\mathrm{nzr}}, J_{\mathrm{nzr}}}$, and its network $f^{\sharp}_{J_{\mathrm{nzr}}}$ coincides with the trained network $\widehat{f}$ for training data, i.e., for $i=1,\cdots,n$, $t=1,\cdots,T$
\[
f^{\sharp}_{J_{\mathrm{nzr}},t}(X^{i}_{t})=\widehat{f}_{t}(X^{i}_{t})
\]
which means that the trained RNN is compressed to size $m^{\sharp}$ without degradation. On the other hand, in the case of $m^{\sharp}<m_{\mathrm{nzr}}$, $\widetilde{A}_{J}=I_{[m],J}$ is not a solution of the minimization problem for any choice of the index $J$, which means that the compressed network using $\widehat{A}_{J}=\widehat{\Sigma}_{[m],J}\widehat{\Sigma}_{J,J}^{\dag}$ is closer to the trained network than that using $\widetilde{A}_{J}=I_{[m],J}$. Therefore, spectral pruning essentially contributes to compression when $m^{\sharp}<m_{\mathrm{nzr}}$.
\end{rem}
\section{Generalization Error Bounds for Compressed RNNs}\label{Generalization error bounds}
In this section, we discuss the generalization error bounds for compressed RNNs. In Subsection~\ref{The error bound for the general compressed RNN}, the error bounds for general compressed RNNs are evaluated to explain the reason for deriving spectral pruning discussed in Section \ref{Pruning algorithm} in the error bound term. In Subsection \ref{The error bound for RNNs compressed by spectral pruning}, the error bounds for RNNs compressed with spectral pruning are evaluated.
\subsection{Error bound for general compressed RNNs}\label{The error bound for the general compressed RNN}
Let $(X^{i}_{T}, Y^{i}_{T})$ be the training data generated independently identically from the true distribution $P_T$, and let $f^{\sharp}$ be a general compressed RNN, and assume that it belongs to the following function space:
\[
\begin{split}
\mathcal{F}^{\sharp}_{T} &= \mathcal{F}^{\sharp}_{T}(R_o, R_h, R_i, R^{b}_{o},R^{b}_{hi})\\
&
:= \bigg\{f^{\sharp}\, \Big|\, f^{\sharp}(X_T)=(f^{\sharp}_{t}(X_t))_{t=1}^{T},\\
& f^{\sharp}_{t}(X_t)=(W^{\sharp o}\sigma(\cdot) +b^{\sharp o})\circ(W^{\sharp h}\sigma(\cdot)+W^{\sharp i}x_{t}+b^{\sharp hi} ) \circ \\
&  \cdots \circ(W^{\sharp h}\sigma(\cdot)+W^{\sharp i}x_{2}+b^{\sharp hi} ) \circ (W^{\sharp i}x_{1}+b^{\sharp hi} ) \\
&
\text{ for }X_T \in \mathrm{supp}(P_{X_T}),\ \big\|W^{\sharp o} \big\|_{F} \leq R_{o},\ \big\|W^{\sharp h} \big\|_{F} \leq R_{h}, \\
& \big\|W^{\sharp i} \big\|_{F} \leq R_{i}, \
\big\|b^{\sharp o} \big\|_{2} \leq R^{b}_{o},\ \big\|b^{\sharp hi} \big\|_{2} \leq R^{b}_{hi} \bigg\},
\end{split}
\]
where $P_{X_T}$ is the marginal distribution of $P_T$ with respect to $X_T$, and $R_{o}$, $R_{h}$, $R_{i}$, $R^{b}_{o}$, $R^{b}_{hi}$ are the upper bounds of the compressed weights $W^{\sharp o} \in \mathbb{R}^{d_y \times m^{\sharp}}$, $W^{\sharp h} \in \mathbb{R}^{m^{\sharp} \times m^{\sharp}}$, $W^{\sharp i} \in \mathbb{R}^{m^{\sharp} \times d_x}$, biases $b^{\sharp o} \in \mathbb{R}^{d_y}$, and $b^{\sharp hi} \in \mathbb{R}^{m^{\sharp}}$, respectively. Here, $\| \cdot \|_{F}$ denotes the Frobenius norm.
\begin{ass}\label{assmption1}
The following assumptions are made: {\bf (i)} The marginal distribution $P_{x_t}$ of $P_T$ with respect to $x_t$ is bounded, i.e., there exist a constant $R_x$ independent of $t$ such that $\|x_{t} \|_{2} \leq R_{x}$ for all $x_t \in \mathrm{supp}(P_{x_t})$ and $t=1,\ldots,T.$
{\bf (ii)} The activation function $\sigma:\mathbb{R} \to \mathbb{R}$ satisfies $\sigma(0)=0$ and $|\sigma(t) - \sigma(s)| \leq \rho_{\sigma} |t-s|$ for all $t, s \in \mathbb{R}.$

\end{ass}
Under these assumptions, we obtain the following approximation error bounds between the trained network $\widehat{f}$ and compressed networks $f^{\sharp}$.
\begin{prop}\label{approximation error bound}
Let Assumption \ref{assmption1} hold. Let $(X_{T}^{1}, Y_{T}^{1}),\ldots,(X_{T}^{n}, Y_{T}^{n})$ be sampled i.i.d. from the distribution $P_T$. Then, for all $f^{\sharp} \in \mathcal{F}^{\sharp}_{T}$ and $J \subset [m]$ with $|J|=m^{\sharp}$, we have
\begin{equation}\label{approximation error bound-1}
\begin{split}
&\big\|\widehat{f}-f^{\sharp}\big\|_{n,T}
 \le \big\| \widehat{W}^{o}\phi-W^{\sharp o}\phi_{J} \big\|_{n,T}\\
+ & \, \big\| \widehat{W}^{h}_{J,[m]}\phi-W^{\sharp h}\phi_{J} \big\|_{n,T} + \big\|\widehat{W}^{i}_{J,[d_x]} - W^{\sharp i}\big\|_{op}  \\
 + & \, \big\| \widehat{b}^{hi}_{J} - b^{\sharp hi} \big\|_{2} + \big\|\widehat{b}^{o}-b^{\sharp o}\big\|_{2}.
\end{split}
\end{equation}
\end{prop}
Here, $\lesssim$ implies that the left-hand side in (\ref{approximation error bound-1}) is bounded above by the right-hand side times a constant independent of the trained weights and biases $\widehat{W}$, $\widehat{b}$ and compressed weights and biases $W^{\sharp}$, $b^{\sharp}$. The proof is given by direct computations. For the exact statement and proof, see Appendix \ref{proof of approximation}. 
\begin{rem}\label{approx is bounded by infor}
Let $f^{\sharp}_{J}$ be the network compressed using the reconstruction matrix (see (iii) in Section~\ref{Pruning algorithm}). By applying Proposition \ref{approximation error bound} as $f^{\sharp}=f^{\sharp}_{J}$, we obtain
\begin{equation}\label{approx<infor}
\begin{split}
& \big\|\widehat{f}-f^{\sharp}_{J}\big\|^{2}_{n,T} \lesssim  
\underbrace{\big\| \widehat{W}^{o}\phi-\widehat{W}^{o}\widehat{A}_{J}\phi_{J} \big\|^{2}_{n,T}+\big\|\widehat{W}^{o}\widehat{A}_{J}\big\|^{2}_{\tau} }_{=L^{(B,o)}_{\tau}(J)}\\
&
+ \underbrace{\big\| \widehat{W}^{h}_{J,[m]}\phi-\widehat{W}^{h}_{J,[m]}\widehat{A}_{J}\phi_{J} \big\|^{2}_{n,T}+\big\|\widehat{W}^{h}_{J,[m]}\widehat{A}_{J}\big\|^{2}_{\tau}}_{=L^{(B,h)}_{\tau}(J)}\\
& \le 
\left(\big\| \widehat{W}^{o} \big\|^{2}_{F}+\big\| \widehat{W}^{h}_{J,[m]}\big\|^{2}_{F}\right) \underbrace{\left(\big\|\phi-\widehat{A}_{J}\phi_{J} \big\|^{2}_{n,T}+\big\|\widehat{A}_{J}\big\|^{2}_{\tau}\right)}_{= L^{(A)}_{\tau}(J)}, 
\end{split}
\end{equation}
i.e., the approximation error is bounded above by the input information loss.
\end{rem}
\par
For the RNN $f=(f_{t})_{t=1}^{T}$, the training error with respect to the $j$-th component of the output is defined as
\[
\widehat{\Psi}_{j}(f):=\frac{1}{nT} \sum_{i=1}^{n} \sum_{t=1}^{T}\psi(y^{i}_{t,j}, f_{t}(X^{i}_{t})_{j}),
\]
where $X_{t}=(x_{t})_{t=1}^{t}$ and $\psi: \mathbb{R} \times \mathbb{R} \to \mathbb{R}_{+}$ is a loss function. The generalization error with respect to the $j$-th component of the output is defined as
\[
\Psi_{j}(f):=E\bigg[\frac{1}{T}\sum_{t=1}^{T} \psi(y_{t,j}, f_{t}(X_{t})_{j})\bigg],
\]
where the expectation is taken with respect to $(X_T, Y_T) \sim P_T$. 
\begin{ass}\label{assmption2}
The following assumptions are made: {\bf (i)} The loss function $\psi(y_{t,j}, 0)$ is bounded, i.e.,  there exists a constant $R_{y}$ such that $|\psi(y_{t,j}, 0)| \leq R_y$ for all $y_{t,j} \in \mathrm{supp}(P_{y_{t,j}})$, $t=1,\ldots,T$, $j=1,\ldots,d_y$.
{\bf (ii)} $\psi$ is $\rho_{\psi}$-Lipschitz continuous, i.e., 
$|\psi(y, f) - \psi(y, g)| \leq \rho_{\psi} |f-g|$ for all $y, f, g \in \mathbb{R}.$
\end{ass}
We obtain the following generalization error bound for $f^{\sharp} \in \mathcal{F}_{T}^{\sharp}(R_o, R_h, R_i, R^{b}_{o},R^{b}_{hi})$.
\begin{thm}\label{generalization error bound}
Let Assumptions \ref{assmption1} and \ref{assmption2} hold, and let $(X_{T}^{1}, Y_{T}^{1}),\ldots,(X_{T}^{n}, Y_{T}^{n})$ be sampled i.i.d. from the distribution $P_T$. Then, for any $\delta \geq \log2$, we have the following inequality with probability greater than $1-2e^{-\delta}$:
\begin{equation}\label{general error bound-1}
\begin{split}
 & \Psi_{j}(f^{\sharp})
\lesssim \widehat{\Psi}_{j}(\widehat{f}) + \Big\{
\big\| \widehat{W}^{o}\phi-W^{\sharp o}\phi_{J} \big\|_{n,T}\\
& + 
 \big\| \widehat{W}^{h}_{J,[m]}\phi-W^{\sharp h}\phi_{J}\big\|_{n,T} +\big\|\widehat{W}^{i}_{J,[d_x]} - W^{\sharp i}\big\|_{op} 
\\
& + 
\big\| \widehat{b}^{hi}_{J} - b^{\sharp hi} \big\|_{2} + \big\|\widehat{b}^{o}-b^{\sharp o}\big\|_{2} \Big\} +  \frac{1}{\sqrt{n}}(m^{\sharp})^{\frac{5}{4}} R^{1/2}_{\infty, T},
\end{split}
\end{equation}
for $j=1,\ldots,d_y$ and for all $J \subset [m]$ with $|J|=m^{\sharp}$, and $f^{\sharp} \in \mathcal{F}^{\sharp}_{T}$,
where $R_{\infty, t}$ is defined by
\[
R_{\infty, t}:=R_{o}\rho_{\sigma}(R_{i}R_{x}+R^{b}_{hi})\bigg(\sum_{l=1}^{t}(R_{h}\rho_{\sigma})^{l-1} \bigg)+R^{b}_{o}.
\]
\end{thm}
Here, $\lesssim$ implies that the left-hand side in (\ref{general error bound-1}) is bounded above by the right-hand side times a constant independent of the trained weights and biases $\widehat{W}$, $\widehat{b}$, compressed weights and biases $W^{\sharp}$, $b^{\sharp}$, compressed number $m^{\sharp}$, and the number of samples $n$. We remark that some omitted constants blow up as increasing $T$, but they can be controlled by increasing sampling number $n$ (see Theorem \ref{thm:C}). The idea behind the proof is that the generalization error is decomposed into the training, approximation, and estimation errors. The approximation and estimation errors are evaluated using Proposition \ref{approximation error bound} and the estimation of the {\it Rademacher complexity}, respectively. For the exact statement and proof, see Appendix \ref{proof of generalization}. 
\par
The second term in (\ref{general error bound-1}) is the approximation error bound between $\widehat{f}$ and $f^{\sharp}$ regarded as the {\it bias}, which is given by Proposition \ref{approximation error bound-1}, while the third term is the estimation error bound regarded as the {\it variance}. It can be observed that minimizing the terms $\| \widehat{W}^{o}\phi-W^{\sharp o}\phi_{J} \|_{n,T}$ and $\| \widehat{W}^{h}_{J,[m]}\phi-W^{\sharp h}\phi_{J} \|_{n,T}$ with respect to $W^{\sharp o}$ and $W^{\sharp h}$ is equivalent to the output information losses (\ref{output information loss-o}) and (\ref{output information loss-h}) with $\tau=0$, respectively, which means that (iii) in Section \ref{Pruning algorithm} with $\tau=0$ constructs the compressed RNN such that the bias term becomes smaller. Considering $\tau\neq0$ prevents the blow up of $\|W^{\sharp o}\|_{F}$ and $\|W^{\sharp h}\|_{F}$, which means that the regularization parameter $\tau$ plays an important role in preventing the blow up of the variance term because $R_{\infty, T}$ in the variance term includes the upper bounds $R_{o}$ and $R_{h}$ of $\|W^{\sharp o}\|_{F}$ and $\|W^{\sharp h}\|_{F}$. Therefore, (iii) with $\tau\neq0$ constructs the compressed RNN such that the generalization error bound becomes smaller. In addition, selecting an optimal $J$ for minimizing the information losses (see (iv) in Section \ref{Pruning algorithm}) further decreases the error bound. 
\subsection{Error bound for RNNs compressed with spectral pruning}\label{The error bound for RNNs compressed by spectral pruning}
Next, we evaluate the generalization error bounds for the RNN $f^{\sharp}_{J}$ compressed using the reconstruction matrix (see (iii) in Section \ref{Pruning algorithm}). We define the {\it degrees of freedom} $\widehat{N}(\lambda)$ by
\[
\widehat{N}(\lambda):=\mathrm{Tr}\big[ \widehat{\Sigma}( \widehat{\Sigma}+\lambda I )^{-1} \big]=\sum_{j=1}^{m}\frac{\widehat{\mu}_{j}}{\widehat{\mu}_{j}+\lambda},
\]
where $\widehat{\mu}_{j}$ is an eigenvalue of $\widehat{\Sigma}$. Throughout this subsection, the regularization parameter $\tau \in \mathbb{R}^{m^{\sharp}}_{+}$ is chosen as $\tau=\lambda m^{\sharp} \tau^{\prime}$, where $\lambda>0$ satisfies 
\begin{equation}
m^{\sharp} \geq 5 \widehat{N}(\lambda) \log(16\widehat{N}(\lambda)/\widetilde{\delta}), \label{Bach condition}
\end{equation}
for a prespecified $\widetilde{\delta} \in (0, 1/2)$. Here, $\tau^{\prime}=(\tau^{\prime}_{j})_{j \in J}\in \mathbb{R}^{m^{\sharp}}$ is the {\it leverage score} defined by for $ k \in [m]$ 
\begin{equation}\label{leverage}
\tau^{\prime}_{k}:=\frac{1}{\widehat{N}(\lambda)}\big[ \widehat{\Sigma}( \widehat{\Sigma}+\lambda I )^{-1} \big]_{k,k}=\frac{1}{\widehat{N}(\lambda)}\sum_{j=1}^{m}U^{2}_{k,j}\frac{\widehat{\mu}_{j}}{\widehat{\mu}_{j}+\lambda},
\end{equation}
where $U=(U_{k,j})_{k,j}$ is the orthogonal matrix that diagonalizes $\widehat{\Sigma}$, i.e., $\widehat{\Sigma}=U\text{diag}\left\{\widehat{\mu}_{1},\ldots, \widehat{\mu}_{m}\right\}U^{T}$. The leverage score includes the information of the eigenvalues and eigenvectors of $\widehat{\Sigma}$, and indicates that the large components correspond to the important nodes from the viewpoint of the spectral information of $\widehat{\Sigma}$. Let $q$ be the probability measure on $[m]$ defined by
\begin{equation}\label{distribution-q}
q(v):=\tau^{\prime}_{v}\quad \text{for} \ v \in [m].
\end{equation}
\begin{prop}\label{bound of input infor}
Let $v_1,\ldots,v_{m^{\sharp}}$ be sampled i.i.d. from the distribution $q$ in \eqref{distribution-q}, and $J=\{v_1,\ldots,v_{m^{\sharp}} \}$. Then, for any $\widetilde{\delta} \in (0,1/2)$ and $\lambda>0$ satisfying (\ref{Bach condition}), we have the following inequality with probability greater than $1-\widetilde{\delta}$:
\begin{equation}
L^{(A)}_{\tau}(J)\leq 4 \lambda. \label{input infor<lambda}
\end{equation}
\end{prop}
The proof is given in Appendix \ref{proof of Bach}. In the proof, we essentially refer to previous work [\cites{Bach}. Combining (\ref{approx<infor}) and (\ref{input infor<lambda}), we conclude that
\begin{equation}
\big\|\widehat{f}-f^{\sharp}_{J}\big\|^{2}_{n,T} \lesssim \lambda. \label{approx<lambda}
\end{equation}
It can be observed that the approximation error bound (\ref{approx<lambda}) is controlled by the degrees of freedom. If the eigenvalues of $\widehat{\Sigma}$ rapidly decrease, then $\widehat{N}(\lambda)$ is a rapidly decreasing function as $\lambda$ is large. Therefore, in that case, we can choose a smaller $\lambda$ even when $m^{\sharp}$ is fixed. We will numerically study the relationship between the eigenvalue distribution and the input information loss in Section \ref{eigenvalue distributions and information losses}.
\par
We make the following additional assumption.
\begin{ass}\label{assmption3}
Assume that the upper bounds for the trained weights and biases are given by $\|\widehat{W}^{o} \|_{F} \leq \widehat{R}_{o}$, $\|\widehat{W}^{h} \|_{F} \leq \widehat{R}_{h}$, $\|\widehat{W}^{i} \|_{F} \leq \widehat{R}_{i}$, $\|\widehat{b}^{hi} \|_{2} \leq \widehat{R}^{b}_{hi}$, and $\|\widehat{b}^{o} \|_{2} \leq \widehat{R}^{b}_{o}$.
\end{ass}
We have the following generalization error bound.
\begin{thm}\label{generalization error bound for SP}
Let Assumptions \ref{assmption1}, \ref{assmption2}, and \ref{assmption3} hold, and let $(X_{T}^{1}, Y_{T}^{1}),\ldots,(X_{T}^{n}, Y_{T}^{n})$ and $v_1,\ldots,v_{m^{\sharp}}$ be sampled i.i.d. from the distributions $P_T$ and $q$ in \eqref{distribution-q}, respectively. Let $J=\{v_1,\ldots,v_{m^{\sharp}} \}$. Then, for any $\delta \geq \log2$ and  $\widetilde{\delta} \in (0,1/2)$, we have the following inequality with probability greater than $(1-2e^{-\delta})\widetilde{\delta}$:
\begin{equation}
\Psi_{j}(f^{\sharp}_{J})
\lesssim \widehat{\Psi}_{j}(\widehat{f})+\sqrt{\lambda}+\frac{1}{\sqrt{n}}(m^{\sharp})^{\frac{5}{4}},
\label{general error bound-2}
\end{equation}
for $j=1,\ldots,d_y$ and for all $\lambda>0$ satisfying \eqref{Bach condition}.
\end{thm}
Here, $\lesssim$ implies that the left-hand side in (\ref{general error bound-2}) is bounded above by the right-hand side times a constant independent of $\lambda$, $m^{\sharp}$, and $n$. We remark that some omitted constants blow up as increasing $T$, but they can be controlled by increasing sampling number $n$ (see Theorem \ref{thm:E}). The proof is given by the combination of applying Theorem~\ref{generalization error bound} as $f^{\sharp}=f^{\sharp}_{J}$ and using Proposition \ref{bound of input infor}. For the exact statement and proof, see Appendix \ref{proof of generalization SP}. It can be observed that in (\ref{general error bound-2}), a bias-variance tradeoff relationship exists with respect to $m^{\sharp}$. When $m^{\sharp}$ is large, $\lambda$ can be chosen smaller in the condition (\ref{Bach condition}), which implies that the bias term (the second term in (\ref{general error bound-2})) becomes smaller, but the variance term (the third term in (\ref{general error bound-2})) becomes larger. In contrast, the bias becomes larger and the variance becomes smaller when $m^{\sharp}$ is small. 
Further remarks on Theorem \ref{generalization error bound for SP} are given in Appendix \ref{Remarks for}.
\section{Numerical Experiments} \label{Numerical Experiments}
In this section, numerical experiments are detailed to demonstrate our theoretical results and show the effectiveness of spectral pruning compared with existing methods.
In Sections \ref{eigenvalue distributions and information losses} and \ref{Pixel-MNIST}, we select the pixel-MNIST as our task and employ the IRNN, which is an RNN that uses the ReLU as the activation function and initializes weights as the identity matrix and biases to zero (see [\cites{Le}). 
In Section \ref{Language Modeling}, we select the PTB [\cites{marcus1993building} and employ the RNNLM whose RNN layer is orthodox Elman-type.
For RNN training details, see Appendix \ref{Detailed configurations for training, pruning and fine-tuning}. 
We choose parameters $\theta_1 = 1$, $\theta_2=\theta_3=0$ in (iv) of Section \ref{Pruning algorithm}, i.e., we minimize only the input information loss. 
This choice is not so problematic because the bound of output information loss  automatically becomes smaller with minimizing the input one (see Remark \ref{approx is bounded by infor}). 
We choose the regularization parameter $\tau=0$, where this choice regards $\widehat{f}$ as a well-trained network and gives priority to minimizing the approximation error between $\widehat{f}$ and $f^{\sharp}_{J}$ (see below Theorem\ref{generalization error bound}).
\subsection{Eigenvalue distributions and information losses}\label{eigenvalue distributions and information losses}
First, we numerically study the relationship between the eigenvalue distribution and the information losses. 
Figure \ref{eigenvalue distributions} shows the eigenvalue distribution of the covariance matrix $\widehat{\Sigma}$ with 128 hidden nodes, which are sorted in decreasing order. 
In this experiment, almost half of the eigenvalues are zero, which cannot be visualized in the figure. Figure \ref{Information losses} shows the input information loss $L^{(A)}_{0}(J)$ versus the compressed number $m^{\sharp}$. 
The information losses vanish when $m^{\sharp}>m_{\mathrm{nzr}}$ (see Remark~\ref{regularization parameter is zero}). 
The blue and pink curves correspond to MNIST\footnote{http://yann.lecun.com/exdb/mnist/} and FashionMNIST\footnote{https://github.com/zalandoresearch/fashion-mnist}, respectively. 
It can be observed that the eigenvalues for MNIST decrease more rapidly than those for FashionMNIST, and the information losses for MNIST decrease more rapidly than those for FashionMNIST. 
This phenomenon coincides with the interpretation on Proposition \ref{bound of input infor} (see the discussion below (\ref{approx<lambda})). 
\begin{figure}[htbp]
\hspace{-5mm}
\begin{tabular}{c}
\begin{minipage}{0.45\hsize}
\vspace{0mm}
  \begin{center}
   \includegraphics[scale=0.45]{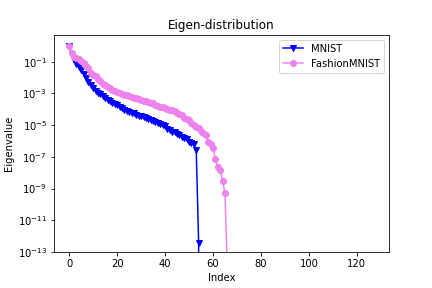}
  \end{center}
  \subcaption{Eigenvalue distribution for $\widehat{\Sigma}$}\label{eigenvalue distributions}
 \end{minipage}
\hspace{10mm}
 \begin{minipage}{0.45\hsize}
 \begin{center}
  \includegraphics[scale=0.45]{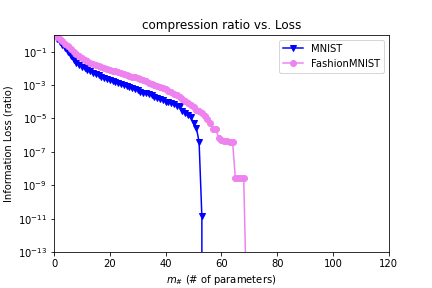}
 \end{center}
 \subcaption{Input information loss vs. $m^{\sharp}$}\label{Information losses}
 \end{minipage}
\end{tabular}
\caption{Relationship between the eigenvalue distribution and the input information loss}\label{eigenvalue and information}
\end{figure}
\subsection{Pixel-MNIST (IRNN)}\label{Pixel-MNIST}
\begin{table*}[h]
\caption{Pixel-MNIST (IRNN)}\label{acc Pixel-MNIST}
\centering
\begin{tabular}{@{}lc@{\hspace{-0.5mm}}c@{\hspace{-0.5mm}}c@{\hspace{-0.5mm}}c@{\hspace{-0.5mm}}c@{\hspace{-0.5mm}}c@{\hspace{-0.5mm}}c@{}}
\toprule
Method  &  Accuracy[\%] (std) & \begin{tabular}{c} Finetuned \\ Accuracy[\%](std) \end{tabular} & \begin{tabular}{c} \# input \\ -hidden\end{tabular} & \begin{tabular}{c} \# hidden \\ -hidden\end{tabular} & \begin{tabular}{c} \# hidden \\ -out\end{tabular} & total \\ 
\midrule
Baseline(128) & 96.80 (0.23) & - & 128 & 16384 & 1280 & 17792 \\ 
Baseline(42) & 93.35 (0.75) & - & 42 & 1764 & 420 & 2226 \\ 
\textbf{Spectral w/\ rec.(ours)}  & \textbf{92.61 (2.46)} & \textbf{97.08 (0.16)} & \textbf{42} & \textbf{1764} & \textbf{420} & \textbf{2226} \\
Spectral w/o rec.  & 83.60 (8.24)  & - & 42 & 1764 & 420 & 2226 \\
Random w/\ rec.    & 34.72 (32.47) & - & 42 & 1764 & 420 & 2226 \\ 
Random w/o rec.    & 23.13 (16.09) & - & 42 & 1764 & 420 & 2226 \\ 
Random Weight          & 10.35 (1.38)  & - & 128 & 1764 & 1280 & 3172 \\
Magnitude-based Weight & 11.06 (0.70)  & 94.41 (3.02) & 128 & 1764 & 1280 & 3172 \\
Column Sparsification  & 84.80 (7.29)  & - & 128 & 5376 & 1280 & 6784 \\
Low Rank Factorization & 9.65 (3.85)   & - & 128 & 10752 & 1280 & 12160 \\

\bottomrule
\end{tabular}
\end{table*}
\begin{table*}[h]
\caption{PTB (RNNLM)}\label{acc Language Modeling}
\centering
\begin{tabular}{@{}lc@{\hspace{-0.5mm}}c@{\hspace{-0.5mm}}c@{\hspace{-0.8mm}}c@{\hspace{-1.3mm}}c@{\hspace{0mm}}c@{\hspace{-0.1mm}}c@{}}
\toprule
Method  &  Perplexity (std) & \begin{tabular}{c}  Finetuned \\ Perplexity (std) \end{tabular} & \begin{tabular}{c} \# input \\ -hidden\end{tabular} & \begin{tabular}{c} \# hidden \\ -hidden\end{tabular} & \begin{tabular}{c} \# hidden \\ -out\end{tabular} & total \\ 
\midrule
Baseline(128) & 114.66 (0.35) & - & 1270016 & 16384 & 1270016 & 2556416 \\ 
Baseline(42) & 145.85 (0.74) & 132.46 (0.74) & 416724 & 1764 & 416724 & 835212 \\ 
\textbf{Spectral w/\ rec.(ours)}  & \textbf{207.63 (2.19)} & \textbf{124.26 (0.39)} & \textbf{416724} & \textbf{1764} & \textbf{416724} & \textbf{835212} \\
Spectral w/o rec.  & 433.99 (10.64) & - & 416724 & 1764 & 416724 & 835212 \\
Random w/\ rec.    & 243.76 (9.46)  & - & 416724 & 1764 & 416724 & 835212  \\ 
Random w/o rec.    & 492.06 (22.40) & - & 416724 & 1764 & 416724 & 835212  \\ 
Random Weight          & 203.41 (2.02)  & - & 1270016 & 1764 & 1270016 & 2541796 \\
Magnitude-based Weight & 168.57 (2.57)  & 115.65 (0.31) & 1270016 & 1764 & 1270016 & 2541796 \\
Magnitude-based Weight $\diamondsuit$ & 201.41 (3.60) & 126.20 (0.28) & 416724 & 1764 & 416724 & 835212 \\
Column Sparsification  & 128.98 (0.52)  & - & 1270016 & 5376 & 1270016 & 2545408 \\
Low Rank Factorization & 126.24 (1.79)  & - & 1270016 & 10752 & 1270016 & 2550784 \\
\bottomrule
\end{tabular}
\end{table*}
We compare spectral pruning with other pruning methods in the pixel-MNIST (IRNN). 
Table \ref{acc Pixel-MNIST} summarizes the accuracies and the number of weight parameters for different pruning methods.
We consider one-third compression in the hidden state, i.e., for the node pruning, 128 hidden nodes were compressed to 42 nodes, while for weight pruning, $128^2(=16384)$ hidden weights were compressed to $42^2(=1764)$ weights. 
\par
``Baseline(128)'' and ``Baseline(42)'' represent direct training (not pruning) with 128 and 42 hidden nodes, respectively. 
``Spectral w/ rec.(ours)''  represents spectral pruning with the reconstruction matrix (i.e., the compressed weight is chosen as $W^{\sharp h}=\widehat{W}^{h}_{J,[m]}\widehat{A}_{J}$ with the optimal $J$ with respect to (\ref{optimization-RNN})), while ``Spectral w/o rec.'' represents spectral pruning without the reconstruction matrix (i.e., $W^{\sharp h}=\widehat{W}^{h}_{J,J}$ with the optimal $J$ with respect to (\ref{optimization-RNN})), which idea is based on [\cites{luo2017thinet}.
``Random w/ rec.'' represents random node pruning with the reconstruction matrix (i.e., $W^{\sharp h}=\widehat{W}^{h}_{J,[m]}\widehat{A}_{J}$, where $J$ is randomly chosen), while ``Random w/o rec.'' represents random node pruning without the reconstruction matrix (i.e., $W^{\sharp h}=\widehat{W}^{h}_{J,J}$, where $J$ is randomly chosen).
``Random Weight'' represents random weight pruning. 
For the reason why we compare with random pruning, see the introduction of [\cites{Zhang}.
``Magnitude-based Weight'' represents magnitude-based weight pruning based on [\cites{Narang}.
``Column Sparsification'' represents the magnitude-based column sparsification during training based on [\cites{wang2019acceleration}. 
``Low Rank Factorization'' represents low rank factorization which truncates small singular values of trained weights based on [\cites{prabhavalkar2016compression}.
``Accuracy[\%](std)'' and ``Finetuned Accuracy[\%](std)'' represent their mean (standard deviation) of accuracy before and after fine-tuning, respectively.
``\# input-hidden'', ``\# hidden-hidden'', and ``\# hidden-out'' represent the number of input-to-hidden, hidden-to-hidden, and hidden-to-output weight parameters, respectively.
``total'' represents their sum. 
For detailed procedures of training, pruning, and fine-tuning, see Appendix \ref{Detailed configurations for training, pruning and fine-tuning}. 
\par
We demonstrate that spectral pruning significantly outperforms other pruning methods. 
The reason why spectral pruning can compress with small degradation is that the covariance matrix $\widehat{\Sigma}$ has a small number of non-zero rows (we observed around 50 non-zero rows). 
For the detail of non-zero rows, see Remark \ref{regularization parameter is zero}.
Our method with fine-tuning outperforms ``Baseline(42)'', which means that the spectral pruning gets benefits from over-parameterization [\cites{chang2020provable, zhang2021understanding}. 
Since the magnitude-based weight pruning is the method to require the fine-tuning (e.g., see [\cites{Narang}), we have also compared our method with the magnitude-based weight pruning with fine-tuning, and observed that our method outperforms the magnitude-based weight pruning as well. 
We also remark that our method with fine-tuning overcomes ``Baseline(128)''.
\subsection{PTB (RNNLM)}\label{Language Modeling}
We compare spectral pruning with other pruning methods in the PTB (RNNLM).
Table \ref{acc Language Modeling} summarizes the perplexity and the number of weight parameters for different pruning methods. 
As in Section \ref{Pixel-MNIST}, we consider one-third compression in the hidden state, and how to represent ``Method'' is the same as Table \ref{acc Pixel-MNIST} except for ``Magnitude-based Weight $\diamondsuit$'', which represents the magnitude-based weight pruning for not only hidden-to-hidden weights but also input-to-hidden and hidden-to-out weights so that the number of resultant weight parameters is the same as Spectral w/ rec.(ours).
\par
We demonstrate that our method with fine-tuning outperforms other pruning methods except for magnitude-based Weight pruning. 
Even though "Low Rank Factorization" retains large number of weight parameters, its perplexity is slightly worse than our method with fine-tuning.
On the other hands, our method with fine-tuning can not outperform ``Magnitude-based Weight'', but it can slightly do under the condition of the same number of weight parameters. 
We also remark that our method with fine-tuning overcomes ``Baseline(42)'', although it does not overcome ``Baseline(128)''.
\par
Therefore, we conclude that spectral pruning works well in Elman-RNN, especially in IRNN.
\section*{Future Work}
It would be interesting to extend our work to the long short-term memory (LSTM). The properties of LSTMs are different from those of RNNs in that LSTMs have the gated architectures including product operations, which might require more complicated analysis of the generalization error bounds as compared with RNNs. Hence, the investigation of spectral pruning for LSTMs is beyond the scope of this study and will be the focus of future work.
\subsubsection*{Acknowledgements}
The authors are grateful to Professor Taiji Suzuki for useful discussions and comments on our work. The first author was supported by Grant-in-Aid for JSPS Fellows (No.21J00119), Japan Society for the Promotion of Science. 
\bibliographystyle{plain}
\bibliography{Spectral_Pruning_for_Recurrent_Neural_Networks.bbl}

\begin{bibdiv}
\begin{biblist}

\bib{akpinar2019sample}{article}{
      author={Akpinar, Nil-Jana},
      author={Kratzwald, Bernhard},
      author={Feuerriegel, Stefan},
       title={Sample complexity bounds for recurrent neural networks with
  application to combinatorial graph problems},
        date={2019},
     journal={arXiv preprint arXiv:1901.10289},
}

\bib{Alom}{inproceedings}{
      author={Alom, Md~Zahangir},
      author={Moody, Adam~T},
      author={Maruyama, Naoya},
      author={Van~Essen, Brian~C},
      author={Taha, Tarek~M},
       title={Effective quantization approaches for recurrent neural networks},
organization={IEEE},
        date={2018},
   booktitle={2018 international joint conference on neural networks (ijcnn)},
       pages={1\ndash 8},
}

\bib{Bach}{article}{
      author={Bach, Francis},
       title={On the equivalence between kernel quadrature rules and random
  feature expansions},
        date={2017},
     journal={The Journal of Machine Learning Research},
      volume={18},
      number={1},
       pages={714\ndash 751},
}

\bib{Bartlett}{article}{
      author={Bartlett, Peter},
      author={Foster, Dylan~J},
      author={Telgarsky, Matus},
       title={Spectrally-normalized margin bounds for neural networks},
        date={2017},
     journal={arXiv preprint arXiv:1706.08498},
}

\bib{chang2020provable}{article}{
      author={Chang, Xiangyu},
      author={Li, Yingcong},
      author={Oymak, Samet},
      author={Thrampoulidis, Christos},
       title={Provable benefits of overparameterization in model compression:
  From double descent to pruning neural networks},
        date={2020},
     journal={arXiv preprint arXiv:2012.08749},
}

\bib{Chen}{article}{
      author={Chen, Minshuo},
      author={Li, Xingguo},
      author={Zhao, Tuo},
       title={On generalization bounds of a family of recurrent neural
  networks},
        date={2019},
     journal={arXiv preprint arXiv:1910.12947},
}

\bib{denil2013predicting}{article}{
      author={Denil, Misha},
      author={Shakibi, Babak},
      author={Dinh, Laurent},
      author={Ranzato, Marc'Aurelio},
      author={De~Freitas, Nando},
       title={Predicting parameters in deep learning},
        date={2013},
     journal={arXiv preprint arXiv:1306.0543},
}

\bib{dodge2019rnn}{article}{
      author={Dodge, Jesse},
      author={Schwartz, Roy},
      author={Peng, Hao},
      author={Smith, Noah~A},
       title={Rnn architecture learning with sparse regularization},
        date={2019},
     journal={arXiv preprint arXiv:1909.03011},
}

\bib{Gal}{inproceedings}{
      author={Gal, Yarin},
      author={Ghahramani, Zoubin},
       title={A theoretically grounded application of dropout in recurrent
  neural networks},
        date={2016},
   booktitle={Advances in neural information processing systems},
      editor={Lee, D.},
      editor={Sugiyama, M.},
      editor={Luxburg, U.},
      editor={Guyon, I.},
      editor={Garnett, R.},
      volume={29},
   publisher={Curran Associates, Inc.},
  url={https://proceedings.neurips.cc/paper/2016/file/076a0c97d09cf1a0ec3e19c7f2529f2b-Paper.pdf},
}

\bib{Gine}{book}{
      author={Gin{\'e}, Evarist},
      author={Nickl, Richard},
       title={Mathematical foundations of infinite-dimensional statistical
  models},
   publisher={Cambridge University Press},
        date={2015},
      volume={40},
}

\bib{gockenbach2016linear}{book}{
      author={Gockenbach, Mark~S},
       title={Linear inverse problems and tikhonov regularization},
   publisher={American Mathematical Soc.},
        date={2016},
      volume={32},
}

\bib{grachev2019compression}{article}{
      author={Grachev, Artem~M},
      author={Ignatov, Dmitry~I},
      author={Savchenko, Andrey~V},
       title={Compression of recurrent neural networks for efficient language
  modeling},
        date={2019},
     journal={Applied Soft Computing},
      volume={79},
       pages={354\ndash 362},
}

\bib{inan2016tying}{article}{
      author={Inan, Hakan},
      author={Khosravi, Khashayar},
      author={Socher, Richard},
       title={Tying word vectors and word classifiers: A loss framework for
  language modeling},
        date={2016},
     journal={arXiv preprint arXiv:1611.01462},
}

\bib{joukovsky2021generalization}{incollection}{
      author={Joukovsky, Boris~Joseph},
      author={Mukherjee, Tanmoy},
      author={Deligiannis, Nikos},
      author={others},
       title={Generalization error bounds for deep unfolding rnns},
        date={2021},
   booktitle={Proceedings of machine learning research},
   publisher={Journal of Machine Learning Research},
}

\bib{Kliegl}{article}{
      author={Kliegl, Markus},
      author={Goyal, Siddharth},
      author={Zhao, Kexin},
      author={Srinet, Kavya},
      author={Shoeybi, Mohammad},
       title={Trace norm regularization and faster inference for embedded
  speech recognition rnns},
        date={2017},
     journal={arXiv preprint arXiv:1710.09026},
}

\bib{Le}{article}{
      author={Le, Quoc~V},
      author={Jaitly, Navdeep},
      author={Hinton, Geoffrey~E},
       title={A simple way to initialize recurrent networks of rectified linear
  units},
        date={2015},
     journal={arXiv preprint arXiv:1504.00941},
}

\bib{Ledoux}{book}{
      author={Ledoux, Michel},
      author={Talagrand, Michel},
       title={Probability in banach spaces: isoperimetry and processes},
   publisher={Springer Science \& Business Media},
        date={2013},
}

\bib{lee2018snip}{article}{
      author={Lee, Namhoon},
      author={Ajanthan, Thalaiyasingam},
      author={Torr, Philip~HS},
       title={Snip: Single-shot network pruning based on connection
  sensitivity},
        date={2018},
     journal={arXiv preprint arXiv:1810.02340},
}

\bib{liu2021selfish}{article}{
      author={Liu, Shiwei},
      author={Mocanu, Decebal~Constantin},
      author={Pei, Yulong},
      author={Pechenizkiy, Mykola},
       title={Selfish sparse rnn training},
        date={2021},
     journal={arXiv preprint arXiv:2101.09048},
}

\bib{liu2021efficient}{article}{
      author={Liu, Shiwei},
      author={Ni’mah, Iftitahu},
      author={Menkovski, Vlado},
      author={Mocanu, Decebal~Constantin},
      author={Pechenizkiy, Mykola},
       title={Efficient and effective training of sparse recurrent neural
  networks},
        date={2021},
     journal={Neural Computing and Applications},
       pages={1\ndash 12},
}

\bib{Liu}{inproceedings}{
      author={Liu, Xuan},
      author={Cao, Di},
      author={Yu, Kai},
       title={Binarized lstm language model},
        date={2018},
   booktitle={Proceedings of the 2018 conference of the north american chapter
  of the association for computational linguistics: Human language
  technologies, volume 1 (long papers)},
       pages={2113\ndash 2121},
}

\bib{lobacheva2020structured}{inproceedings}{
      author={Lobacheva, Ekaterina},
      author={Chirkova, Nadezhda},
      author={Markovich, Alexander},
      author={Vetrov, Dmitry},
       title={Structured sparsification of gated recurrent neural networks},
        date={2020},
   booktitle={Proceedings of the aaai conference on artificial intelligence},
      volume={34},
       pages={4989\ndash 4996},
}

\bib{Lobacheva}{article}{
      author={Lobacheva, Ekaterina},
      author={Chirkova, Nadezhda},
      author={Vetrov, Dmitry},
       title={Bayesian sparsification of recurrent neural networks},
        date={2017},
     journal={arXiv preprint arXiv:1708.00077},
}

\bib{luo2017thinet}{inproceedings}{
      author={Luo, Jian-Hao},
      author={Wu, Jianxin},
      author={Lin, Weiyao},
       title={Thinet: A filter level pruning method for deep neural network
  compression},
        date={2017},
   booktitle={Proceedings of the ieee international conference on computer
  vision},
       pages={5058\ndash 5066},
}

\bib{marcus1993building}{article}{
      author={Marcus, Mitchell},
      author={Santorini, Beatrice},
      author={Marcinkiewicz, Mary~Ann},
       title={Building a large annotated corpus of english: The penn treebank},
        date={1993},
}

\bib{molchanov2017variational}{inproceedings}{
      author={Molchanov, Dmitry},
      author={Ashukha, Arsenii},
      author={Vetrov, Dmitry},
       title={Variational dropout sparsifies deep neural networks},
organization={PMLR},
        date={2017},
   booktitle={International conference on machine learning},
       pages={2498\ndash 2507},
}

\bib{Narang}{inproceedings}{
      author={Narang, Sharan},
      author={Diamos, Greg},
      author={Sengupta, Shubho},
      author={Elsen, Erich},
       title={Exploring sparsity in recurrent neural networks},
        date={2017},
   booktitle={5th international conference on learning representations, {ICLR}
  2017, toulon, france, april 24-26, 2017, conference track proceedings},
   publisher={OpenReview.net},
         url={https://openreview.net/forum?id=BylSPv9gx},
}

\bib{narang2017block}{article}{
      author={Narang, Sharan},
      author={Undersander, Eric},
      author={Diamos, Gregory},
       title={Block-sparse recurrent neural networks},
        date={2017},
     journal={arXiv preprint arXiv:1711.02782},
}

\bib{Parashar}{article}{
      author={Parashar, Angshuman},
      author={Rhu, Minsoo},
      author={Mukkara, Anurag},
      author={Puglielli, Antonio},
      author={Venkatesan, Rangharajan},
      author={Khailany, Brucek},
      author={Emer, Joel},
      author={Keckler, Stephen~W},
      author={Dally, William~J},
       title={Scnn: An accelerator for compressed-sparse convolutional neural
  networks},
        date={2017},
     journal={ACM SIGARCH Computer Architecture News},
      volume={45},
      number={2},
       pages={27\ndash 40},
}

\bib{prabhavalkar2016compression}{inproceedings}{
      author={Prabhavalkar, Rohit},
      author={Alsharif, Ouais},
      author={Bruguier, Antoine},
      author={McGraw, Lan},
       title={On the compression of recurrent neural networks with an
  application to lvcsr acoustic modeling for embedded speech recognition},
organization={IEEE},
        date={2016},
   booktitle={2016 ieee international conference on acoustics, speech and
  signal processing (icassp)},
       pages={5970\ndash 5974},
}

\bib{Shi}{inproceedings}{
      author={Shi, Yangyang},
      author={Hwang, Mei-Yuh},
      author={Lei, Xin},
      author={Sheng, Haoyu},
       title={Knowledge distillation for recurrent neural network language
  modeling with trust regularization},
organization={IEEE},
        date={2019},
   booktitle={Icassp 2019-2019 ieee international conference on acoustics,
  speech and signal processing (icassp)},
       pages={7230\ndash 7234},
}

\bib{Suzuki}{inproceedings}{
      author={Suzuki, Taiji},
      author={Abe, Hiroshi},
      author={Murata, Tomoya},
      author={Horiuchi, Shingo},
      author={Ito, Kotaro},
      author={Wachi, Tokuma},
      author={Hirai, So},
      author={Yukishima, Masatoshi},
      author={Nishimura, Tomoaki},
       title={Spectral pruning: Compressing deep neural networks via spectral
  analysis and its generalization error},
        date={20207},
   booktitle={Proceedings of the twenty-ninth international joint conference on
  artificial intelligence, {IJCAI-20}},
      editor={Bessiere, Christian},
   publisher={International Joint Conferences on Artificial Intelligence
  Organization},
       pages={2839\ndash 2846},
        note={Main track},
}

\bib{tang2015pruning}{article}{
      author={Tang, Shijian},
      author={Han, Jiang},
       title={A pruning based method to learn both weights and connections for
  lstm},
        date={2015},
     journal={Proceedings of the Advances in Neural Information Processing
  Systems, NIPS, Montreal, QC, Canada},
       pages={7\ndash 12},
}

\bib{Tang}{inproceedings}{
      author={Tang, Zhiyuan},
      author={Wang, Dong},
      author={Zhang, Zhiyong},
       title={Recurrent neural network training with dark knowledge transfer},
organization={IEEE},
        date={2016},
   booktitle={2016 ieee international conference on acoustics, speech and
  signal processing (icassp)},
       pages={5900\ndash 5904},
}

\bib{Tjandra}{inproceedings}{
      author={Tjandra, Andros},
      author={Sakti, Sakriani},
      author={Nakamura, Satoshi},
       title={Compressing recurrent neural network with tensor train},
        date={2017},
   booktitle={2017 international joint conference on neural networks, {IJCNN}
  2017, anchorage, ak, usa, may 14-19, 2017},
   publisher={{IEEE}},
       pages={4451\ndash 4458},
         url={https://doi.org/10.1109/IJCNN.2017.7966420},
}

\bib{tu2019understanding}{inproceedings}{
      author={Tu, Zhuozhuo},
      author={He, Fengxiang},
      author={Tao, Dacheng},
       title={Understanding generalization in recurrent neural networks},
        date={2019},
   booktitle={International conference on learning representations},
}

\bib{wang2019acceleration}{article}{
      author={Wang, Shaorun},
      author={Lin, Peng},
      author={Hu, Ruihan},
      author={Wang, Hao},
      author={He, Jin},
      author={Huang, Qijun},
      author={Chang, Sheng},
       title={Acceleration of lstm with structured pruning method on fpga},
        date={2019},
     journal={IEEE Access},
      volume={7},
       pages={62930\ndash 62937},
}

\bib{Wen2020StructuredPO}{article}{
      author={Wen, Liangjiang},
      author={Zhang, Xueyang},
      author={Bai, Haoli},
      author={Xu, Zenglin},
       title={Structured pruning of recurrent neural networks through neuron
  selection},
        date={2020},
     journal={Neural networks : the official journal of the International
  Neural Network Society},
      volume={123},
       pages={134\ndash 141},
}

\bib{wen2017learning}{article}{
      author={Wen, Wei},
      author={He, Yuxiong},
      author={Rajbhandari, Samyam},
      author={Zhang, Minjia},
      author={Wang, Wenhan},
      author={Liu, Fang},
      author={Hu, Bin},
      author={Chen, Yiran},
      author={Li, Hai},
       title={Learning intrinsic sparse structures within long short-term
  memory},
        date={2017},
     journal={arXiv preprint arXiv:1709.05027},
}

\bib{Zaremba}{article}{
      author={Zaremba, Wojciech},
      author={Sutskever, Ilya},
      author={Vinyals, Oriol},
       title={Recurrent neural network regularization},
        date={2014},
     journal={arXiv preprint arXiv:1409.2329},
}

\bib{zhang2021understanding}{article}{
      author={Zhang, Chiyuan},
      author={Bengio, Samy},
      author={Hardt, Moritz},
      author={Recht, Benjamin},
      author={Vinyals, Oriol},
       title={Understanding deep learning (still) requires rethinking
  generalization},
        date={2021},
     journal={Communications of the ACM},
      volume={64},
      number={3},
       pages={107\ndash 115},
}

\bib{Zhang}{article}{
      author={Zhang, Matthew~Shunshi},
      author={Stadie, Bradly},
       title={One-shot pruning of recurrent neural networks by jacobian
  spectrum evaluation},
        date={2019},
     journal={arXiv preprint arXiv:1912.00120},
}

\end{biblist}
\end{bibdiv}
\newpage
\onecolumn
\appendix

\part*{Appendix}
\section{Review of Spectral Pruning for DNNs}\label{Review of the spectral pruning}
Let $D=\{(x^{i}, y^{i})\}_{i=1}^{n}$ be training data, where $x^{i} \in \mathbb{R}^{d_x}$ is an input and $y^{i} \in \mathbb{R}^{d_y}$ is an output. The training data are independently identically distributed. To train the appropriate relationship between input
and output, we consider DNNs $f$ as
\[
f(x)=(W^{(L)}\sigma(\cdot) +b^{(L)})\circ \cdots \circ(W^{(1)}x+b^{(1)} ),
\]
where $\sigma:\mathbb{R} \to \mathbb{R}$ is an activation function, $W^{(l)} \in \mathbb{R}^{m_{l+1} \times m_{l}}$ is a weight matrix, and $b^{(l)} \in \mathbb{R}^{m_{l+1}}$ is a bias. Let $\widehat{f}$ be a trained DNN obtained from the training data $D$, i.e.,
\[
\widehat{f}(x)=(\widehat{W}^{(L)}\sigma(\cdot) + \widehat{b}^{(L)})\circ \cdots \circ(\widehat{W}^{(1)}x+\widehat{b}^{(1)}).
\]
We denote the input with respect to $l$-th layer 
by
\[
\phi^{(l)}(x)=\sigma \circ (\widehat{W}^{(l-1)}\sigma(\cdot)+\widehat{b}^{(l-1)})\circ \cdots \circ(\widehat{W}^{(1)}x+\widehat{b}^{(1)}).
\]
\par
Let $J^{(l)} \subset [m_{l}]$ be an index set with $|J^{(l)}|=m^{\sharp}_{l}$, where $[m_{l}]:=\{1,\ldots,m_{l} \}$ and $m^{\sharp}_{l} \in \mathbb{N}$ is the number of nodes of the $l$-th layer of the compressed DNN $f^{\sharp}$ with $m^{\sharp}_{l} \leq m_{l}$. Let $\phi^{(l)}_{J^{(l)}}(x)=(\phi^{(l)}_{j}(x))_{j \in J^{(l)}}$ be a subvector of $\phi^{(l)}(x)$ corresponding to the index set $J^{(l)}$, where $\phi^{(l)}_{j}(x)$ is the $j$-th components of the vector $\phi^{(l)}(x)$.
\par
{\bf (i) Input information loss.} 
The input information loss is defined by
\begin{equation}
L^{(A,l)}_{\tau}(J^{(l)}):= \min_{A \in \mathbb{R}^{m_{l} \times m^{\sharp}_{l} } } \big\{ \big\| \phi^{(l)} - A \phi^{(l)}_{J^{(l)}} \big\|^{2}_{n} + \big\| A \big\|^{2}_{\tau} \big\}, \label{input information loss-DNN}
\end{equation}
where $\| \cdot \|_{n}$ is the empirical $L^2$-norm with respect to $n$, i.e., 
\begin{equation}\label{formula-A}
\big\| \phi^{(l)} - A\phi^{(l)}_{J^{(l)}} \big\|^{2}_{n}:=\frac{1}{n} \sum_{i=1}^{n} \big\| \phi^{(l)}(x^{i}) - A \phi^{(l)}_{J^{(l)}}(x^{i}) \big\|^{2}_{2},
\end{equation}
where $\| \cdot \|_{2}$ is the Euclidean norm and $\| A \|^{2}_{\tau}:=\mathrm{Tr}\, [AI_{\tau}A^{T} ]$ for a regularization parameter $\tau \in \mathbb{R}^{m^{\sharp}_{l}}_{+}$. Here, $\mathbb{R}^{m^{\sharp}_{l}}_{+}:=\big\{x \in \mathbb{R}^{m^{\sharp}_{l}}\, |\, x_{j}>0, \ j=1,\ldots,m^{\sharp}_{l} \big\}$ and $I_{\tau}:=\text{diag}(\tau)$. By the linear regularization theory, there exists a unique solution $\widehat{A}^{(l)}_{J^{(l)}} \in \mathbb{R}^{m_{l} \times m^{\sharp}_{l}}$ of the minimization problem of $\| \phi^{(l)} - A \phi^{(l)}_{J^{(l)}}\|^{2}_{n} + \| A\|^{2}_{\tau}$, and it has the form
\begin{equation}
\widehat{A}^{(l)}_{J^{(l)}}=\widehat{\Sigma}^{(l)}_{[m_{l}],J^{(l)}}\big(\widehat{\Sigma}^{(l)}_{J^{(l)},J^{(l)}}+I_{\tau} \big)^{-1}, \label{A^(l)_J} 
\end{equation}
where $\widehat{\Sigma}^{(l)}$ is the (noncentered) empirical covariance matrix of $\phi^{(l)}(x)$ with respect to $n$, i.e., 
\[
\widehat{\Sigma}^{(l)}=\frac{1}{n}\sum_{i=1}^{n}\phi^{(l)}(x^{i})\phi^{(l)}(x^{i})^{T},
\] 
and $\widehat{\Sigma}^{(l)}_{I,I'} = (\widehat{\Sigma}^{(l)}_{i,i'})_{i\in I,i'\in I'}\in \mathbb{R}^{K \times H}$ is the submatrix of $\widehat{\Sigma}^{(l)}$ corresponding to index sets $I, I' \subset [m]$ with $|I|=K$ and $|I'|=H$.
By substituting the
explicit formula \eqref{formula-A} of the reconstruction matrix $\widehat{A}^{(l)}_{J^{(l)}}$ into (\ref{input information loss-DNN}), the input information loss is reformulated as
\begin{equation}
L^{(A,l)}_{\tau}(J^{(l)})=\mathrm{Tr}\Big[ \widehat{\Sigma}^{(l)} - \widehat{\Sigma}^{(l)}_{[m_{l}],J^{(l)}}\big( \widehat{\Sigma}^{(l)}_{J^{(l)},J^{(l)}}+I_{\tau} \big)^{-1}\widehat{\Sigma}^{(l)}_{J^{(l)},[m_{l}]} \Big].
\end{equation}
\par
{\bf (ii) Output information loss.} For any matrix $Z^{(l)} \in \mathbb{R}^{m \times m_{l}}$ with an output size $m \in \mathbb{N}$, we define the output information loss by
\begin{equation}
L^{(B,l)}_{\tau}(J^{(l)}):=\sum_{j=1}^{m}\min_{\beta \in \mathbb{R}^{m_{l}^{\sharp}}} \big\{ \big\| Z^{(l)}_{j,:} \phi^{(l)} - \beta^{T} \phi^{(l)}_{J^{(l)}} \big\|^{2}_{n} + \big\| \beta^{T} \big\|^{2}_{\tau} \big\}, \label{output information loss-DNN}
\end{equation}
where $Z^{(l)}_{j,:}$ denotes the $j$-th row of the matrix $Z^{(l)}$. A typical situation is that $Z^{(l)}=\widehat{W}^{(l)}$. The minimization problem of $\| Z^{(l)}_{j,:} \phi^{(l)} - \beta^{T} \phi_{J^{(l)}}\|^{2}_{n}$ $+$ $\| \beta^{T} \|^{2}_{\tau}$
has the unique solution 
\[
\widehat{\beta}^{(l)}_{j}=(Z^{(l)}_{j,:}\widehat{A}^{(l)}_{J^{(l)}})^{T},
\]
and by substituting it into \eqref{output information loss-DNN}, the output information loss is reformulated as
\[
L^{(B,l)}_{\tau}(J^{(l)}) = \mathrm{Tr}\Big[ Z^{(l)} \big(\widehat{\Sigma}^{(l)} - \widehat{\Sigma}^{(l)}_{[m],J^{(l)}}\big( \widehat{\Sigma}^{(l)}_{J^{(l)},J^{(l)}}+I_{\tau} \big)^{-1}\widehat{\Sigma}^{(l)}_{J^{(l)},[m]} \big)Z^{(l)T} \Big].
\]
\par
{\bf (iii) Compressed DNN by the reconstruction matrix.} We construct the compressed DNN by 
\[
f^{\sharp}_{J^{(1:L)}}(x)=(W^{\sharp (L)}_{J^{(L)}}\sigma(\cdot) +b^{\sharp(L)})\circ \cdots \circ(W^{\sharp(1)}_{J^{(1)}}x+b^{\sharp(1)}),
\] 
where $J^{(1:L)}=J^{(1)}\cup \cdots \cup J^{(L)}$, and $b^{\sharp(l)} = \widehat{b}^{(l)}$ and $W^{\sharp(l)}_{J^{(l)}}$ is the compressed weight as 
the multiplication of the trained weight $\widehat{W}^{(l)}_{J^{(l+1)},[m_{l}]}$ and the reconstruction matrix $\widehat{A}_{J^{(l)}}$, i.e.,
\begin{equation}\label{compressed_weight}
W^{\sharp(l)}_{J^{(l)}}:=\widehat{W}^{(l)}_{J^{(l+1)},[m_{l}]}\widehat{A}_{J^{(l)}}.
\end{equation}
.
\par
{\bf (iv) Optimization.} 
To select an appropriate index set $J^{(l)}$, we consider the following optimization problem that minimizes a convex combination of input and output information losses, i.e.,
\[
\min_{J^{(l)} \subset [m_{l}]\ s.t.\ |J^{(l)}|=m^{\sharp}_{l}} \big\{ \theta L^{(A,l)}_{\tau}(J^{(l)})+ (1-\theta)L^{(B,l)}_{\tau}(J^{(l)}) \big\},
\]
for $\theta \in [0,1]$, 
where $m^{\sharp}_{l} \in [m]$ is a prespecified number. 
We adapt the optimal index $J^{\sharp (1:L)}$ in the algorithm. We term this method as {\it spectral pruning}.
\par
In [\cites{Suzuki}, the generalization error bounds for compressed DNNs with the spectral pruning have been studied (see Theorems 1 and 2 in [\cites{Suzuki}), and the parameters $\theta$, $\tau$, and $Z^{(l)}$ are chosen such that its error bound become smaller.
\section{Proof of Proposition \ref{approximation error bound}}\label{proof of approximation}
We restate Proposition \ref{approximation error bound} in an exact form as follows:

\begin{prop}\label{prop:B}
Suppose that Assumption \ref{assmption1} holds. Let $\{(X_{T}^{i}, Y_{T}^{i})\}_{i=1}^n$ be sampled i.i.d. from the distribution $P_T$. Then,
\begin{multline}\label{approximation error bound-1-Appendix}
\|\widehat{f}-f^{\sharp}\|_{n,T}
\leq \sqrt{3}\, \biggl\{ \big\| \widehat{W}^{o}\phi-W^{\sharp o}\phi_{J} \big\|_{n,T}
+ R_{o}\rho_{\sigma}\max\{1, (R_{h} \rho_{\sigma})^{T-2}\}T\big\| \widehat{W}^{h}_{J,[m]}\phi-W^{\sharp h}\phi_{J} \big\|_{n,T}\\
+ R_{o}\rho_{\sigma}\bigg(\sum_{t=1}^{T} (R_{h} \rho_{\sigma})^{t-1}\bigg) \big(R_{x}\big\|\widehat{W}^{i}_{J,[d_x]} - W^{\sharp i}\big\|_{op} + \| \widehat{b}^{hi}_{J} - b^{\sharp hi} \|_{2} \big) + \big\|\widehat{b}^{o}-b^{\sharp o}\big\|_{2} \biggr\},
\end{multline}
for all $f^{\sharp} \in \mathcal{F}^{\sharp}_{T}(R_o, R_h, R_i, R^{b}_{o},R^{b}_{hi})$ and $J \subset [m]$ with $|J|=m^{\sharp}$.
\end{prop}
\begin{proof}
Let $\widehat{f}=(\widehat{f}_{t})_{t=1}^{T}$ be a trained RNN and $f^{\sharp} \in \mathcal{F}^{\sharp}_{T}(R_o, R_h, R_i, R^{b}_{o},R^{b}_{hi})$. 
Let us define functions $\phi$ and $\phi^{\sharp}$ by
\[
\phi(x, h):=\sigma(\widehat{W}^{h}h+\widehat{W}^{i}x+\widehat{b}^{hi})\quad \text{for} \ x \in \mathbb{R}^{d_x},\ h \in \mathbb{R}^{m},
\]
\begin{equation}\label{B-phi-sharp}
\phi^{\sharp}(x, h^{\sharp}):=\sigma(W^{\sharp h}h^{\sharp}+W^{\sharp i}x+b^{\sharp hi})\quad \text{for} \ x \in \mathbb{R}^{d_x},\ h^\sharp \in \mathbb{R}^{m^{\sharp}},
\end{equation}
and denote the hidden states by
\begin{equation}\label{B-hiddens}
\widehat{h}_{t}:=\phi(x_t, \widehat{h}_{t-1}), \ \ h^{\sharp}_{t}:=\phi^{\sharp}(x_t, h^{\sharp}_{t-1})\quad \text{for} \ t=1,2,\cdots, T.
\end{equation}
If a training data $X^i_T = (x^{i}_{t})_{t=1}^T$ is used as input, we denote its hidden states by
\[
\widehat{h}^{i}_{t}:=\phi(x^{i}_t, \widehat{h}^{i}_{t-1}), \ \ h^{\sharp i}_{t}:=\phi^{\sharp}(x^{i}_t, h^{\sharp i}_{t-1}),
\]
and its outputs at time $t$ by 
\[
\widehat{f}_{t}(X^{i}_{t})=\widehat{W}^{o} \phi(x^{i}_{t}, \widehat{h}^{i}_{t-1})+\widehat{b}^{o},
\quad 
f^{\sharp}_{t}(X^{i}_{t})=W^{\sharp o}\phi^{\sharp}(x^{i}_{t}, h^{\sharp i}_{t-1})+b^{\sharp o},
\] 
for $t=1,2,\ldots,T$. 
Then, we have
\begin{equation}\label{B-1}
\begin{split}
\big\| \widehat{f}_{t}(X^{i}_{t}) - f^{\sharp}_{t}(X^{i}_{t}) \big\|_{2}
& \le 
\big\| \widehat{W}^{o} \phi(x^{i}_{t}, \widehat{h}^{i}_{t-1})-W^{\sharp o}\phi_{J}(x^{i}_{t}, \widehat{h}^{i}_{t-1}) \big\|_{2}\\
& +
\big\| W^{\sharp o} \phi_{J}(x^{i}_{t}, \widehat{h}^{i}_{t-1})-W^{\sharp o}\phi^{\sharp}(x^{i}_{t}, h^{\sharp i}_{t-1}) \big\|_{2}+\big\|\widehat{b}^{o}-b^{\sharp o}\big\|_{2}.
\end{split}
\end{equation}
If we can prove that the second term of right-hand side in \eqref{B-1} is estimated as 
\begin{multline}\label{B-goal}
\big\| W^{\sharp o} \phi_{J}(x^{i}_{t}, \widehat{h}^{i}_{t-1})-W^{\sharp o}\phi^{\sharp}(x^{i}_{t}, h^{\sharp i}_{t-1}) \big\|_{2}\\
\le 
R_o \rho_{\sigma}
\bigg\{
\max\big\{1,(R_h \rho_{\sigma})^{t-2}\big\} \sum_{l=1}^{t-1} \big\|  \widehat{W}^{h}_{J,[m]}\phi(x^{i}_{t-l}, \widehat{h}^{i}_{t-l-1})-W^{\sharp h}\phi_{J}(x^{i}_{t-l}, \widehat{h}^{i}_{t-l-1}) \big\|_{2}\\
+
\sum_{l=1}^{t}(R_h \rho_{\sigma})^{l-1} \big(\big\| \widehat{W}^{i}_{J,[d_x]}-W^{\sharp i} \big\|_{op}\|x^{i}_{t-l+1}\|_{2}+\big\|\widehat{b}^{hi}_{J}-b^{\sharp hi}\big\|_{2}\big)
\bigg\},
\end{multline}
then by using the inequalities \eqref{B-1} and $(\sum_{k=1}^K a_k)^2 \le K\sum_{k=1}^Ka_k^2$, 
we have
\[
\begin{split}
& \big\| \widehat{f}_{t}(X^{i}_{t}) - f^{\sharp}_{t}(X^{i}_{t}) \big\|_{2}^2
\le 
3 \Bigg\{\big\| \widehat{W}^{o} \phi(x^{i}_{t}, \widehat{h}^{i}_{t-1})-W^{\sharp o}\phi_{J}(x^{i}_{t}, \widehat{h}^{i}_{t-1}) \big\|_{2}^2\\
& +
\bigg(R_o \rho_{\sigma}
\max\big\{1,(R_h \rho_{\sigma})^{t-2}\big\} \sum_{l=1}^{t-1} \big\|  \widehat{W}^{h}_{J,[m]}\phi(x^{i}_{t-l}, \widehat{h}^{i}_{t-l-1})-W^{\sharp h}\phi_{J}(x^{i}_{t-l}, \widehat{h}^{i}_{t-l-1}) \big\|_{2}
\bigg)^2\\
& +
\bigg(
R_o \rho_{\sigma}
\sum_{l=1}^{t}(R_h \rho_{\sigma})^{l-1} \big(\big\| \widehat{W}^{i}_{J,[d_x]}-W^{\sharp i} \big\|_{op}\|x^{i}_{t-l+1}\|_{2}+\big\|\widehat{b}^{hi}_{J}-b^{\sharp hi}\big\|_{2}\big)
+
\big\|\widehat{b}^{o}-b^{\sharp o}\big\|_{2}
\bigg)^2
\Bigg\}.
\end{split}
\]
Hence, by taking the average over $i=1,\ldots,n$ and $t=1,\ldots, T$, and by 
using the inequality $\sum_{t=1}^{T}(\sum_{l=1}^{t}a_{l})^{2} \leq T^{2}\sum_{t=1}^{T}a_{t}^{2}$,
we obtain
\[
\begin{split}
& \big\|\widehat{f}-f^{\sharp}\big\|_{n,T}^2
=
\frac{1}{nT}\sum_{i=1}^{n}\sum_{t=1}^{T}\big\| \widehat{f}_{t}(X^{i}_{t}) - f^{\sharp}_{t}(X^{i}_{t}) \big\|^{2}_{2} \\
& \le 3\, \Bigg\{
\underbrace{\frac{1}{nT}\sum_{i=1}^{n}\sum_{t=1}^{T}\big\| \widehat{W}^{o} \phi(x^{i}_{t}, \widehat{h}^{i}_{t-1})-W^{\sharp o}\phi_{J}(x^{i}_{t}, \widehat{h}^{i}_{t-1}) \big\|_{2}^2}_
{=\| \widehat{W}^{o}\phi-W^{\sharp o}\phi_{J}\|_{n,T}^2}\\
& +
\big(R_o \rho_{\sigma} \max\big\{1,(R_h \rho_{\sigma})^{T-2}\big\} T\big)^2
\underbrace{\frac{1}{nT}\sum_{i=1}^{n}\sum_{t=1}^{T}
\big\|  \widehat{W}^{h}_{J,[m]}\phi(x^{i}_{t}, \widehat{h}^{i}_{t-1})-W^{\sharp h}\phi_{J}(x^{i}_{t}, \widehat{h}^{i}_{t-1}) \big\|_{2}^2}_{= \| \widehat{W}^{h}_{J,[m]}\phi-W^{\sharp h}\phi_{J} \|_{n,T}^2}\\
& +
\bigg(
R_{o}\rho_{\sigma}\bigg(\sum_{t=1}^{T} (R_{h} \rho_{\sigma})^{t-1}\bigg) \big(\big\| \widehat{W}^{i}_{J,[d_x]}-W^{\sharp i} \big\|_{op}\|x^{i}_{t-l+1}\|_{2}+\big\|\widehat{b}^{hi}_{J}-b^{\sharp hi}\big\|_{2}\big)
+\big\|\widehat{b}^{o}-b^{\sharp o}\big\|_{2}^2\bigg)^2
\Bigg\},
\end{split}
\]
which concludes the inequality \eqref{approximation error bound-1-Appendix}. 
It remains to prove \eqref{B-goal}. 
We calculate that
\begin{equation}\label{B-3}
\begin{split}
 \big\| W^{\sharp o} \phi_{J}(x^{i}_{t}, &\, \widehat{h}^{i}_{t-1}) -W^{\sharp o}\phi^{\sharp}(x^{i}_{t}, h^{\sharp i}_{t-1}) \big\|_{2}\\ 
& \le 
\big\|W^{\sharp o}\big\|_{op} \big\| \sigma\big(\widehat{W}^{h}_{J,[m]}\phi(x^{i}_{t-1}, \widehat{h}^{i}_{t-2}) + \widehat{W}^{i}_{J,[d_x]}x^{i}_{t}+\widehat{b}^{hi}_{J}\big)\\
& \hspace{4cm}-\sigma\big(W^{\sharp h}\phi^{\sharp}(x^{i}_{t-1}, h^{\sharp i}_{t-2}) + W^{\sharp i}x^{i}_{t}+b^{\sharp hi}\big) \big\|_{2}\\
& \le
R_{o}\rho_{\sigma}\bigg\{ \underbrace{\big\| \widehat{W}^{h}_{J,[m]}\phi(x^{i}_{t-1}, \widehat{h}^{i}_{t-2})-W^{\sharp h}\phi^{\sharp}(x^{i}_{t-1}, h^{\sharp i}_{t-2}) \big\|_{2}}_{=: H_{t-1}}\\
& \hspace{4cm}+\big\| \widehat{W}^{i}_{J,[d_x]}-W^{\sharp i} \big\|_{op}\|x^{i}_{t}\|_{2}+\big\|\widehat{b}^{hi}-b^{\sharp hi}\big\|_{2} \bigg\}, 
\end{split}
\end{equation}
where $\| \cdot \|_{op}$ is the operator norm (which is the largest singular value). 
Concerning the quantity $H_{t-1}$, we estimate
\begin{multline*}
H_{t-1}
\le
\big\|\widehat{W}^{h}_{J, [m]} \phi(x^{i}_{t-1}, \widehat{h}^{i}_{t-2})-W^{\sharp h}\phi_{J}(x^{i}_{t-1}, \widehat{h}^{i}_{t-2})\big\|_{2}\\
+\big\| W^{\sharp h} \phi_{J}(x^{i}_{t-1}, \widehat{h}^{i}_{t-2})-W^{\sharp h}\phi^{\sharp}(x^{i}_{t-1}, h^{\sharp i}_{t-2}) \big\|_{2},
\end{multline*}
and moreover, the second term is estimated as 
\[
\begin{split}
\big\| W^{\sharp h} \phi_{J}(x^{i}_{t-1}, &\, \widehat{h}^{i}_{t-2})-W^{\sharp h}\phi^{\sharp}(x^{i}_{t-1}, h^{\sharp i}_{t-2}) \big\|_{2}\\
& \le
\big\| W^{\sharp h} \big\|_{op}
\big\|\sigma\big(\widehat{W}^{h}_{J,[m]}\phi(x^{i}_{t-2}, \widehat{h}^{i}_{t-3}) + \widehat{W}^{i}_{J,[d_x]}x^{i}_{t-1}+\widehat{b}^{hi}_{J}\big)\\
& \quad\quad\quad\quad\quad\quad\quad\quad\quad\quad\quad -\sigma\big(W^{\sharp h}\phi^{\sharp}(x^{i}_{t-2}, h^{\sharp i}_{t-3}) + W^{\sharp i}x^{i}_{t-1}+b^{\sharp hi}\big) \big\|_{2}\\
& \le
R_h \rho_{\sigma}\Big\{ H_{t-2} + \big\|\widehat{W}^{i}_{J,[d_x]}- W^{\sharp i} \big\|_{op}\|x^{i}_{t-1}\|_{2}+\big\|\widehat{b}^{hi}_{J} - b^{\sharp hi} \big\|_{2} \Big\},
\end{split}
\]
for all $t$. 
Thus, we have the recursive inequality 
\begin{equation}\label{B-4}
\begin{split}
H_{t-1}
& \le
\big\|\widehat{W}^{h}_{J, [m]} \phi(x^{i}_{t-1}, \widehat{h}^{i}_{t-2})-W^{\sharp h}\phi_{J}(x^{i}_{t-1}, \widehat{h}^{i}_{t-2})\big\|_{2}\\
& \quad +
R_h \rho_{\sigma}\Big\{ H_{t-2} + \big\|\widehat{W}^{i}_{J,[d_x]}- W^{\sharp i} \big\|_{op} \|x^{i}_{t-1}\|_{2}+\big\|\widehat{b}^{hi}_{J} - b^{\sharp hi} \big\|_{2} \Big\},
\end{split}
\end{equation}
for $t = 2,\ldots, T$. By repeatedly substituting \eqref{B-4} into \eqref{B-3}, we arrive at \eqref{B-goal}:
\[
\begin{split}
\big\| W^{\sharp o} &\,\phi_{J}(x^{i}_{t}, \widehat{h}^{i}_{t-1})-W^{\sharp o}\phi^{\sharp}(x^{i}_{t}, h^{\sharp i}_{t-1}) \big\|_{2}\\
& \le 
R_o \rho_{\sigma}
\bigg\{
\sum_{l=1}^{t-1} \underbrace{(R_h \rho_{\sigma})^{l-1}}_{\le\, \max \{1,(R_h \rho_{\sigma})^{t-2}\} } 
\big\|  \widehat{W}^{h}_{J,[m]}\phi(x^{i}_{t-l}, \widehat{h}^{i}_{t-l-1})-W^{\sharp h}\phi_{J}(x^{i}_{t-l}, \widehat{h}^{i}_{t-l-1}) \big\|_{2}\\
& \quad +
\sum_{l=1}^{t}(R_h \rho_{\sigma})^{l-1} \big(\big\| \widehat{W}^{i}_{J,[d_x]}-W^{\sharp i} \big\|_{op}\|x^{i}_{t-l+1}\|_{2}+\big\|\widehat{b}^{hi}_{J}-b^{\sharp hi}\big\|_{2}\big)
\bigg\}.
\end{split}
\]
Thus, we conclude Proposition~\ref{prop:B}. 
\end{proof}
\section{Proof of Theorem \ref{generalization error bound}}\label{proof of generalization}

We restate Theorem \ref{generalization error bound} in an exact form as follows:

\begin{thm}\label{thm:C}
Suppose that Assumptions \ref{assmption1} and \ref{assmption2} hold. Let $\{(X_{T}^{i}, Y_{T}^{i})\}_{i=1}^n$ be sampled i.i.d. from the distribution $P_T$. 
Then, for any $\delta \geq \log2$, we have the following inequality with probability greater than $1-2e^{-\delta}$:
\[
\begin{split}
\Psi_{j}(f^{\sharp})
& \le 
\widehat{\Psi}_{j}(\widehat{f})
+
\sqrt{3}\rho_{\psi} \Bigg\{
\big\| \widehat{W}^{o}\phi-W^{\sharp o}\phi_{J} \big\|_{n,T}\\
& \quad + R_{o}\rho_{\sigma}\max\{1, (R_{h} \rho_{\sigma})^{T-2}\}T\big\| \widehat{W}^{h}_{J,[m]}\phi-W^{\sharp h}\phi_{J} \big\|_{n,T}\\
& \quad + R_{o}\rho_{\sigma}\bigg(\sum_{t=1}^{T} (R_{h} \rho_{\sigma})^{t-1}\bigg) \big(R_{x}\big\|\widehat{W}^{i}_{J,[d_x]} - W^{\sharp i}\big\|_{op} + \big\| \widehat{b}^{hi}_{J} - b^{\sharp hi} \big\|_{2} \big) + \big\|\widehat{b}^{o}-b^{\sharp o}\big\|_{2} \Bigg\}\\ 
& \quad + 
\frac{1}{\sqrt{n}}\Bigg\{\frac{\widehat{c}\rho_{\psi}\sqrt{m^{\sharp}}}{T}\bigg(\sum_{t=1}^{T}M_{t}^{1/2}R^{1/2}_{\infty, t}\bigg)  +3\sqrt{2\delta} (\rho_{\psi}R_{\infty, T}+R_{y}) \Bigg\},
\end{split}
\]
for $j=1,\ldots,d_y$ and for all  $J \subset [m]$ with $|J|=m^{\sharp}$ and $f^{\sharp} \in \mathcal{F}^{\sharp}_{T}(R_o, R_h, R_i, R^{b}_{o},R^{b}_{hi})$,
where $\widehat{c}:=192\sqrt{5}$, and $R_{\infty, t}$ and $M_t$ are defined by
\begin{equation}\label{def of R_infty-Appendix}
R_{\infty, t}:=R_{o}\rho_{\sigma}(R_{i}R_{x}+R^{b}_{hi})\bigg(\sum_{l=1}^{t}(R_{h}\rho_{\sigma})^{l-1} \bigg)+R^{b}_{o},
\end{equation}
\begin{equation}\label{def of M_t-Appendix}
\begin{split}
M_{t} & :=R_{o}\rho_{\sigma}\biggl[\big( d_y\min\{\sqrt{m^{\sharp}}, \sqrt{d_y}\} + d_x\min\{\sqrt{m^{\sharp}}, \sqrt{d_x}\} \big)R_{i}R_{x}\\
& \quad +
\big(d_y\min\{\sqrt{m^{\sharp}}, \sqrt{d_y}\} + 1\big)R^{b}_{hi} \biggr] \bigg(\sum_{l=0}^{t-1} (R_{h}\rho_{\sigma})^{l}\bigg)\\
& \quad +
(m^{\sharp})^{\frac{3}{2}}R_{h}\rho_{\sigma}^{2}R_{o}(R_{i}R_{x}+R^{b}_{hi})\bigg(\sum_{l=1}^{t-1}\sum_{k=0}^{l-1} (R_{h}\rho_{\sigma})^{t-1-l+k}\bigg)+d_{y}R^{b}_{o}.
\end{split}
\end{equation}
\end{thm}
\begin{proof} 
The generalization error of $f^{\sharp}_{t} \in \mathcal{F}^{\sharp}_{t}$ is decomposed into
\[
\Psi_{j}(f^{\sharp})=\Psi_{j}(\widehat{f})+\big(\widehat{\Psi}_{j}(f^{\sharp}) - \widehat{\Psi}_{j}(\widehat{f}) \big) + \big( \Psi_{j}(f^{\sharp})-\widehat{\Psi}_{j}(f^{\sharp}) \big) ,
\]
where the second term $\widehat{\Psi}_{j}(f^{\sharp}) - \widehat{\Psi}_{j}(\widehat{f})$ is called the approximation error and the third term $\Psi_{j}(f^{\sharp})-\widehat{\Psi}_{j}(f^{\sharp})$ is called the estimation error. Since the loss function $\psi$ is $\rho_{\psi}$-Lipschitz continuous, the approximation error is evaluated as
\[
\begin{split}
\big|\widehat{\Psi}_{j}(f^{\sharp}) - \widehat{\Psi}_{j}(\widehat{f})\big|
& \le 
\frac{1}{nT}\sum_{i=1}^{n}\sum_{t=1}^{T}\big| \psi(y^{i}_{t,j}, f^{\sharp}_{t}(X^{i}_{t})_{j}) - \psi(y^{i}_{t}, \widehat{f}_{t}(X^{i}_{t})_{j}) \big| \\
& \le 
\frac{\rho_{\psi}}{nT}\sum_{i=1}^{n}\sum_{t=1}^{T}\big| f^{\sharp}(X^{i}_{t})_{j} - \widehat{f}_{t}(X^{i}_{t})_{j} \big|\\
& \le 
\frac{\rho_{\psi}}{nT}\sum_{i=1}^{n}\sum_{t=1}^{T}\big\| f^{\sharp}_{t}(X^{i}_{t}) - \widehat{f}_{t}(X^{i}_{t}) \big\|_{2}\\
& \le
\rho_{\psi}\sqrt{\frac{1}{nT}\sum_{i=1}^{n}\sum_{t=1}^{T}\big\| f^{\sharp}_{t}(X^{i}_{t}) - \widehat{f}_{t}(X^{i}_{t}) \big\|^{2}_{2}}=\rho_{\psi}\big\|f^{\sharp}-\widehat{f}\big\|_{n,T}.
\end{split}
\]
The term $\|f^{\sharp}-\widehat{f}\|_{n,T}$ is evaluated by Proposition \ref{approximation error bound} (see also Proposition \ref{prop:B}). In the rest of the proof, let us concentrate on the estimation error bound.
\par
First, we define the following function space 
\[
\mathcal{G}^{\sharp}_{T,j}:=\bigg\{
g_{j}
\, \bigg| \, 
g_{j}(Y_T, X_T)=\frac{1}{T}\sum_{t=1}^{T}\psi(y_{t,j},f_{t}(X_{t})_{j}) \text{ for $(X_T, Y_T) \in \mathrm{supp}(P_T)$, $f \in \mathcal{F}^{\sharp}_{T}$},
\bigg\}
\]
for $j=1,\ldots,d_{y}$. For $g_{j} \in \mathcal{G}^{\sharp}_{T,j}$, we have 
\[
\begin{split}
\big|g_{j}(Y_T, X_T)\big|
& \le 
\frac{1}{T}\sum_{t=1}^{T}\big\{|\psi(y_{t,j},f_{t}(X_{t})_{j})-\psi(y_{t,j},0)|+|\psi(y_{t,j},0)|\big\}\\
& \le 
\frac{1}{T}\sum_{t=1}^{T}\big(\rho_{\psi}|f_{t}(X_{t})_{j}| +R_{y} \big) \leq \frac{\rho_{\psi}}{T}\sum_{t=1}^{T}\|f_{t}(X_t) \|_{2} + R_{y}.
\end{split}
\]
The quantity $\|f_{t}(X_t) \|_{2}$ is evaluated by
\[
\|f_{t}(X_t)\|_{2}
\le
R_o \rho_\sigma \big\|W^{\sharp h}\phi^{\sharp}(x_{t-1}, h^{\sharp i}_{t-2})+W^{\sharp i}x_{t}+b^{\sharp hi} \big\|_{2} + R_o^b.
\]
The recurrent structure \eqref{B-phi-sharp} and \eqref{B-hiddens} give
\[
\begin{split}
\big\|W^{\sharp h}\phi^{\sharp}&\,(x_{t-1}, h^{\sharp i}_{t-2})+W^{\sharp i}x_{t}+b^{\sharp hi} \big\|_{2}\\
& \le R_h \rho_\sigma  \big\|W^{\sharp h}\phi^{\sharp}(x_{t-2}, h^{\sharp i}_{t-3})+W^{\sharp i}x_{t-1}+b^{\sharp hi}\big\|_{2} 
+ R_i R_x +R^{b}_{hi},
\end{split}
\]
as this is repeated, 
\[
\big\|W^{\sharp h}\phi^{\sharp}(x_{t-1}, h^{\sharp i}_{t-2})+W^{\sharp i}x_{t}+b^{\sharp hi} \big\|_{2}
\le (R_{i}R_{x}+R^{b}_{hi}) \bigg(\sum_{l=1}^{t}(R_{h}\rho_{\sigma})^{l-1} \bigg). 
\]
Hence, we see from \eqref{def of R_infty-Appendix} that
\[
\|f_{t}(X_t)\|_{2} \le R_{o}\rho_{\sigma}(R_{i}R_{x}+R^{b}_{hi})\bigg(\sum_{l=1}^{t}(R_{h}\rho_{\sigma})^{l-1} \bigg)+R^{b}_{o} = R_{\infty, t},
\]
which implies that
\[
\begin{split}
\big|g_{j}(Y_T, X_T)\big|
& \le 
\rho_{\psi}R_{o}\rho_{\sigma}(R_{i}R_{x}+R^{b}_{hi})\bigg\{\frac{1}{T}\sum_{t=1}^{T}\sum_{l=1}^{t}(R_{h}\rho_{\sigma})^{l-1} \bigg\}+R^{b}_{o} + R_{y}\\
& \le 
\rho_{\psi}R_{o}\rho_{\sigma}(R_{i}R_{x}+R^{b}_{hi})\bigg(\sum_{t=1}^{T}(R_{h}\rho_{\sigma})^{t-1} \bigg)+R^{b}_{o} + R_{y}\\
& = \rho_{\psi}R_{\infty,T}+R_{y}.
\end{split}
\]
By Theorem 3.4.5 in [\cites{Gine}, for any $\delta>\log2$, we have the following inequality with probability grater than $1-2e^{-\delta}$:
\[
\begin{split}
\big| \Psi_{j}(f^{\sharp})-\widehat{\Psi}_{j}(f^{\sharp}) \big|
& \le 
\sup_{g_{j} \in \mathcal{G}^{\sharp}_{T,j}}\bigg| \frac{1}{n}\sum_{i=1}^{n}g_{j}(Y_{T}^{i},X_{T}^{i})-E_{P_T}[g_{j}(Y_{T},X_{T})]\bigg|\\
& \le 
2E_{\epsilon}\Bigg[ \sup_{g_{j} \in \mathcal{G}^{\sharp}_{T,j}}\bigg| \frac{1}{n}\sum_{i=1}^{n}\epsilon_{i}g_{j}(Y_{T}^{i},X_{T}^{i})\bigg| \Bigg] + 3(\rho_{\psi}R_{\infty, T}+R_{y})\sqrt{\frac{2\delta}{n}},
\end{split}
\]
where $(\epsilon_{i})_{i=1}^{n}$ is the i.i.d. Rademacher sequence (see, e.g., Definition 3.1.19 in [\cites{Gine}). The first term of right-hand side in the above inequality, called the Rademacher complexity, is estimated by using Theorem 4.12 in [\cites{Ledoux}, and Lemma A.5 in [\cites{Bartlett} (or Lemma 9 in [\cites{Chen}) as follows: 
\[
\begin{split}
E_{\epsilon}\Bigg[\sup_{g_{j} \in \mathcal{G}^{\sharp}_{T,j}}\bigg| \frac{1}{n}\sum_{i=1}^{n}\epsilon_{i}g_{j}(Y_{T}^{i},X_{T}^{i})\bigg|\Bigg]
& \le 
\frac{1}{T}\sum_{t=1}^{T}E_{\epsilon}\Bigg[\sup_{f_{t,j} \in \mathcal{F}^{\sharp}_{t,j}}\bigg| \frac{1}{n}\sum_{i=1}^{n}\epsilon_{i}\psi_{j}(y_{t,j}^{i},f_{t}(X_{t}^{i})_{j})\bigg|\Bigg]\\
& \le 
\frac{2\rho_{\psi}}{T}\sum_{t=1}^{T}E_{\epsilon}\Bigg[\sup_{f_{t,j} \in \mathcal{F}^{\sharp}_{t,j}}\bigg| \frac{1}{n}\sum_{i=1}^{n}\epsilon_{i}f_{t}(X_{t}^{i})_{j} \bigg|\Bigg]\\
& \le 
\frac{2\rho_{\psi}}{T}\sum_{t=1}^{T} \inf_{\alpha>0}\bigg( \frac{4\alpha}{\sqrt{n}} + \frac{12}{n}\int_{\alpha}^{2R_{\infty,t}\sqrt{n}} \sqrt{\log N(\mathcal{F}^{\sharp}_{t,j}, \epsilon, \|\cdot\|_{S} ) }\, d\epsilon \bigg),
\end{split}
\]
where $\mathcal{F}^{\sharp}_{t,j}$ and $\| \cdot \|_{S}$ are defined by 
\[
\mathcal{F}^{\sharp}_{t,j}
:=\big\{ f_{t,j} \, \big|\, 
f_{t,j} (X_t) = f_{t}(X_{t})_{j} \text{ for $X_{t} \in \mathrm{supp}(P_{X_t})$, $f \in \mathcal{F}^{\sharp}_{T}$} \big\},
\]
\[
\| f_{t,j} \|_{S}:=\bigg( \sum_{i=1}^{n} |f_{t}(X_{t}^{i})_{j}|^{2} \bigg)^{1/2}.
\]
Here, we denote by $N(F, \epsilon, \|\cdot\|)$ the covering number of $F$ which means 
the minimal cardinality of a subset $C \subset F$ that covers $F$ at scale $\epsilon$ with respect to the norm $\|\cdot\|$. 
By using Lemma \ref{covering number bound} in Appendix \ref{Upper bound for the covering number}, for any $\delta>\log2$, we conclude the following estimation error bound:
\[
\begin{split}
& \big| \Psi_{j}(f^{\sharp})-\widehat{\Psi}_{j}(f^{\sharp}) \big|\\
& \le 
\frac{16\rho_{\psi}\alpha}{\sqrt{n}} + \frac{48\rho_{\psi}}{nT}\sum_{t=1}^{T} \int_{\alpha}^{2R_{\infty,t}\sqrt{n}} \sqrt{\log N(\mathcal{F}^{\sharp}_{t,j}, \epsilon, \|\cdot\|_{S} ) }\,d\epsilon 
+3(\rho_{\psi}R_{\infty, T}+R_{y})\sqrt{\frac{2\delta}{n}}\\
& \le 
\frac{48\rho_{\psi}}{nT}\sqrt{10m^{\sharp}}n^{1/4}\sum_{t=1}^{T}M_{t}^{1/2} \int_{\alpha}^{2R_{\infty,t}\sqrt{n}} \frac{d\epsilon}{\sqrt{\epsilon}} 
+3(\rho_{\psi}R_{\infty, T}+R_{y})\sqrt{\frac{2\delta}{n}}+O(\alpha)\\
& = 
\frac{\widehat{c}\rho_{\psi}\sqrt{m^{\sharp}}}{\sqrt{n}T}\bigg(\sum_{t=1}^{T}M_{t}^{1/2}R^{1/2}_{\infty, t}\bigg)  +3(\rho_{\psi}R_{\infty, T}+R_{y})\sqrt{\frac{2\delta}{n}}+O(\alpha),
\end{split}
\]
for all $\alpha>0$ 
with probability grater than $1-2e^{-\delta}$, where $\widehat{c}:=192\sqrt{5}$, and $M_t$ is defined by \eqref{def of M_t-Appendix}.
The proof of Theorem \ref{thm:C} is complete.
\end{proof}
\section{Upper Bound of the Covering Number}\label{Upper bound for the covering number}
\begin{lem}\label{covering number bound}
Under the same assumptions as in Theorem \ref{thm:C}, the covering number $N(\mathcal{F}^{\sharp}_{t,j}, \epsilon, \|\cdot\|_{S} )$ has the following bound:
\[
\log N(\mathcal{F}^{\sharp}_{t,j}, \epsilon, \|\cdot\|_{S} ) \leq \frac{10m^{\sharp}n^{1/2}M_t}{\epsilon},
\]
for any $\epsilon>0$, 
where $M_t$ is given by \eqref{def of M_t-Appendix}. 
\end{lem}
\begin{proof} 
The proof is based on the argument of proof of Lemma 3 in [\cites{Chen}.
For $f^{\sharp}_{t,j}, \widetilde{f}^{\sharp}_{t,j} \in \mathcal{F}^{\sharp}_{t,j}$, we estimate 
\[
\begin{split}
| f^{\sharp}_{t}(X_t)_{j}- \widetilde{f}^{\sharp}_{t}(X_t)_{j}| 
& \leq \big\| f^{\sharp}_{t}(X_t)- \widetilde{f}^{\sharp}_{t}(X_t) \big\|_{2}\\
& \le 
\big\| W^{\sharp o} -\widetilde{W}^{\sharp o} \big\|_{op} \big\|\phi^{\sharp}(x_{t}, h^{\sharp}_{t-1}) \big\|_{2}\\
& \qquad + \big\| \widetilde{W}^{\sharp o}\widetilde{\phi}^{\sharp}(x_{t}, \widetilde{h}^{\sharp}_{t-1}) -\widetilde{W}^{\sharp o}\phi^{\sharp}(x_{t}, h^{\sharp}_{t-1}) \big\|_{2} +\big\|b^{\sharp o} - \widetilde{b}^{\sharp o} \big\|_{2}.
\end{split}
\]
The second term of right-hand side is estimated as
\begin{multline*}
\big\| \widetilde{W}^{\sharp o}\widetilde{\phi}^{\sharp}(x_{t}, \widetilde{h}^{\sharp}_{t-1}) -\widetilde{W}^{\sharp o}\phi^{\sharp}(x_{t}, h^{\sharp}_{t-1}) \big\|_{2}\\
\le 
\big\|\widetilde{W}^{\sharp o}\big\|_{op} \rho_{\sigma} 
\Big( \big\|\widetilde{W}^{\sharp h}\widetilde{\phi}^{\sharp}(x_{t-1},\widetilde{h}^{\sharp}_{t-2}) - W^{\sharp h}\phi^{\sharp}(x_{t-1}, h^{\sharp}_{t-2})\big\|_{2} \\
+\big\| W^{\sharp i} -\widetilde{W}^{\sharp i} \big\|_{op}\big\|x_{t}\big\|_{2} +\big\|b^{\sharp hi} -\widetilde{b}^{hi}\big\|_{2} \Big).
\end{multline*}
We estimate the first term of right-hand side in the above inequality as 
\[
\begin{split}
& \big\|\widetilde{W}^{\sharp h}\widetilde{\phi}^{\sharp}(x_{t-1}, \widetilde{h}^{\sharp}_{t-2}) - W^{\sharp h}\phi^{\sharp}(x_{t-1}, h^{\sharp}_{t-2})\big\|_{2}\\
& \le 
\big\| \widetilde{W}^{\sharp h}\big\|_{op} \big\|\widetilde{\phi}^{\sharp}(x_{t-1}, \widetilde{h}^{\sharp}_{t-2}) - \phi^{\sharp}(x_{t-1}, h^{\sharp}_{t-2}) \big\|_{2}+\big\|W^{\sharp h} - \widetilde{W}^{\sharp h}\big\|_{op} \big\|\phi^{\sharp}(x_{t-1}, h^{\sharp}_{t-2}) \big\|_{2}\\
& \le 
\big\| \widetilde{W}^{\sharp h}\big\|_{op}\rho_{\sigma}\Big( 
\big\|\widetilde{W}^{\sharp h}\widetilde{\phi}^{\sharp}(x_{t-2},\widetilde{h}^{\sharp}_{t-3}) - W^{\sharp h}\phi^{\sharp}(x_{t-2}, h^{\sharp}_{t-3})\big\|_{2} \\
& \qquad 
+\big\| W^{\sharp i} -\widetilde{W}^{\sharp i} \big\|_{op}\big\|x_{t-1}\big\|_{2} +
\big\|\widetilde{b}^{hi}-b^{\sharp hi}) \big\|_{2} \Big)+\big\|W^{\sharp h} -\widetilde{W}^{\sharp h}\big\|_{op}\big\|\phi^{\sharp}(x_{t-1}, h^{\sharp}_{t-2}) \big\|_{2},
\end{split}
\]
and as this is repeated, we eventually obtain
\begin{multline*}
\big\| \widetilde{W}^{\sharp o}\widetilde{\phi}^{\sharp}(x_{t}, \widetilde{h}^{\sharp}_{t-1}) -\widetilde{W}^{\sharp o}\phi^{\sharp}(x_{t}, h^{\sharp}_{t-1}) \big\|_{2}\\
\le 
R_{o}\rho_{\sigma}\bigg\{\sum_{l=0}^{t-1} (R_{h}\rho_{\sigma})^{l}\Big(\big\| W^{\sharp i}-\widetilde{W}^{\sharp i}\big\|_{op}\big\|x_{t-l}\big\|_{2}
+
\big\|b^{\sharp hi} - \widetilde{b}^{\sharp hi} \big\|_{2}\Big)\\
+
\sum_{l=1}^{t-1} (R_{h}\rho_{\sigma})^{t-1-l}\big\|\phi^{\sharp}(x_{l},h^{\sharp}_{l-1})\big\|_{2}\big\|W^{\sharp h}-\widetilde{W}^{\sharp h}\big\|_{op}\bigg\}.
\end{multline*}
Summarizing the above, we have
\[
\begin{split}
| f^{\sharp}_{t}(X_t)_{j}- \widetilde{f}^{\sharp}_{t}(X_t)_{j}| 
& \le 
\big\| W^{\sharp o} -\widetilde{W}^{\sharp o} \big\|_{op} \big\|\phi^{\sharp}(x_{t}, h^{\sharp}_{t-1}) \big\|_{2}\\
& 
\quad + R_{o}\rho_{\sigma}\Biggl\{\sum_{l=0}^{t-1} (R_{h}\rho_{\sigma})^{l}\Big(\big\| W^{\sharp i}-\widetilde{W}^{\sharp i}\big\|_{op}\big\|x_{t-l}\big\|_{2}
+
\big\|b^{\sharp hi} - \widetilde{b}^{\sharp hi} \big\|_{2}\Big)\\
& \quad + 
\sum_{l=1}^{t-1} (R_{h}\rho_{\sigma})^{t-1-l}\big\|\phi^{\sharp}(x_{l},h^{\sharp}_{l-1})\big\|_{2}\big\|W^{\sharp h}-\widetilde{W}^{\sharp h}\big\|_{op}\Biggr\}+ \big\|b^{\sharp o} - \widetilde{b}^{\sharp o} \big\|_{2}.
\end{split}
\]
Since 
\[
\begin{split}
\big\| \phi^{\sharp}_{t}(x_t, h^{\sharp}_{t-1})\big\|_{2} 
& \le 
\rho_{\sigma}\big(\big\| W^{\sharp h} \big\|_{op} \big\|\phi^{\sharp}(x_{t-1}, h^{\sharp}_{t-2}) \big\|_{2}+\big\| W^{\sharp i}\big\|_{op}\big\|x_{t}\big\|_{2}+ \big\|b^{\sharp hi}\big\|_{2}\big)\\
& \le 
\rho_{\sigma}(R_{i}R_{x}+R^{b}_{hi})\sum_{l=0}^{t-1} (R_{h}\rho_{\sigma})^{l},
\end{split}
\]
and
\[
\begin{split}
\sum_{l=1}^{t-1}(R_{h}\rho_{\sigma})^{t-1-l}\big\| \phi^{\sharp}_{t}(x_l, h^{\sharp}_{l-1})\big\|_{2} 
& \le 
\rho_{\sigma}(R_{i}R_{x}+R^{b}_{hi})\sum_{l=1}^{t-1}\sum_{k=0}^{l-1}(R_{h}\rho_{\sigma})^{t-1-l}(R_{h}\rho_{\sigma})^{k}\\
& = 
\rho_{\sigma}(R_{i}R_{x}+R^{b}_{hi})\sum_{l=1}^{t-1}\sum_{k=0}^{l-1}(R_{h}\rho_{\sigma})^{t-1-l+k},
\end{split}
\]
we see that
\begin{equation}\label{5L}
\begin{split}
& | f^{\sharp}_{t}(X_t)_{j} - \widetilde{f}^{\sharp}_{t}(X_t)_{j}| 
\leq \underbrace{\rho_{\sigma}(R_{i}R_{x}+R^{b}_{hi})\bigg(\sum_{l=0}^{t-1} (R_{h}\rho_{\sigma})^{l}\bigg)}_{=:L_{o,t}}\big\| W^{\sharp o} -\widetilde{W}^{\sharp o} \big\|_{op} \\
& \quad + 
\underbrace{\rho_{\sigma}R_{o}R_{x}\bigg(\sum_{l=0}^{t-1} (R_{h}\rho_{\sigma})^{l}\bigg)}_{=:L_{i,t}}\big\| W^{\sharp i} -\widetilde{W}^{\sharp i} \big\|_{op}+\underbrace{\rho_{\sigma}R_{o}\bigg(\sum_{l=0}^{t-1} (R_{h}\rho_{\sigma})^{l}\bigg)}_{=:L_{b,t}}\big\| b^{\sharp hi} -\widetilde{b}^{\sharp hi} \big\|_{2}\\
& \quad + 
\underbrace{\rho_{\sigma}^{2}R_{o}(R_{i}R_{x}+R^{b}_{hi})\bigg(\sum_{l=1}^{t-1}\sum_{k=0}^{l-1} (R_{h}\rho_{\sigma})^{t-1-l+k}\bigg)}_{=:L_{h,t}}\big\| W^{\sharp h} -\widetilde{W}^{\sharp h} \big\|_{op} 
+ 
\big\| b^{\sharp o} -\widetilde{b}^{\sharp o} \big\|_{2}.
\end{split}
\end{equation}
Since the right-hand side of \eqref{5L} is independent of training data $X_{t}^{i}$, we estimate
\[
\begin{split}
\big\|f^{\sharp}_{t}(X_t)_{j}- \widetilde{f}^{\sharp}_{t}(X_t)_{j}\big\|_{S}
& =\bigg(\sum_{i=1}^{n}\big|f^{\sharp}_{t}(X^{i}_t)_{j}- \widetilde{f}^{\sharp}_{t}(X^{i}_t)_{j}\big|^{2} \bigg)^{1/2}\\
& \le 
\sqrt{n} \Big(
L_{o,t}\big\| W^{\sharp o} -\widetilde{W}^{\sharp o} \big\|_{op} +L_{i,t}\big\| W^{\sharp i} -\widetilde{W}^{\sharp i} \big\|_{op}\\
& \quad + 
L_{b,t}\big\| b^{\sharp hi} -\widetilde{b}^{\sharp hi} \big\|_{2}+
L_{h,t}\big\| W^{\sharp h} -\widetilde{W}^{\sharp h} \big\|_{op} 
+\big\| b^{\sharp o} -\widetilde{b}^{\sharp o} \big\|_{2}\Big).
\end{split}
\]
Then, the covering number $N(\mathcal{F}^{\sharp}_{t,j}, \epsilon, \|\cdot\|_{S} )$ is bounded as follows
\[
\begin{split}
& N(\mathcal{F}^{\sharp}_{t,j}, \epsilon, \|\cdot\|_{S} )
\le N\Big(\mathcal{H}_{W^{\sharp o},R_{o}}, \frac{\epsilon}{5\sqrt{n}L_{o,t}}, \|\cdot\|_{F} \Big)
N\Big(\mathcal{H}_{W^{\sharp i},R_{i}}, \frac{\epsilon}{5\sqrt{n}L_{i,t}}, \|\cdot\|_{F}\Big)\\
&
\times N\Big(\mathcal{H}_{b^{\sharp hi},R^{b}_{hi}}, \frac{\epsilon}{5\sqrt{n}L_{b,t}}, \|\cdot\|_{F}\Big)
 N\Big(\mathcal{H}_{W^{\sharp h},R_{h}}, \frac{\epsilon}{5\sqrt{n}L_{h,t}}, \|\cdot\|_{F}\Big)N\Big(\mathcal{H}_{b^{\sharp o},R^{b}_{o}}, \frac{\epsilon}{5\sqrt{n}}, \|\cdot\|_{F}\Big),
\end{split}
\]
where we used the notation
\[
\mathcal{H}_{A,R}:=\big\{ A\in \mathbb{R}^{d_1 \times d_2}\, |\, \|A \|_{F}\leq R \big\}.
\]
By Lemma 8 in [\cites{Chen}, the above five covering numbers are bounded as 
\[
N\Big(\mathcal{H}_{W^{\sharp o},R_{o}}, \frac{\epsilon}{5\sqrt{n}L_{o,t}}, \|\cdot\|_{F} \Big)
\leq \bigg(1+\frac{10\min\{\sqrt{m^{\sharp}}, \sqrt{d_y}\}R_{o}L_{o,t}\sqrt{n}}{\epsilon}\bigg)^{m^{\sharp}d_y},
\]
\[
N\Big(\mathcal{H}_{W^{\sharp i},R_{i}}, \frac{\epsilon}{5\sqrt{n}L_{i,t}}, \|\cdot\|_{F} \Big)
\leq \bigg(1+\frac{10\min\{\sqrt{m^{\sharp}}, \sqrt{d_x}\}R_{i}L_{i,t}\sqrt{n}}{\epsilon}\bigg)^{m^{\sharp}d_x},
\]
\[
N\Big(\mathcal{H}_{b^{\sharp hi},R^{b}_{hi}}, \frac{\epsilon}{5\sqrt{n}L_{b,t}}, \|\cdot\|_{F} \Big)
\leq \bigg(1+\frac{10R^{b}_{hi}L_{b,t}\sqrt{n}}{\epsilon}\bigg)^{m^{\sharp}},
\]
\[
N\Big(\mathcal{H}_{W^{\sharp h},R_{h}}, \frac{\epsilon}{5\sqrt{n}L_{h,t}}, \|\cdot\|_{F} \Big)
\leq \bigg(1+\frac{10\sqrt{m^{\sharp}}R_{h}L_{h,t}\sqrt{n}}{\epsilon}\bigg)^{(m^{\sharp})^{2}},
\]
\[
N\Big(\mathcal{H}_{b^{\sharp o},R^{b}_{o}}, \frac{\epsilon}{5\sqrt{n}}, \|\cdot\|_{F} \Big)
\leq \bigg(1+\frac{10R^{b}_{o}\sqrt{n}}{\epsilon}\bigg)^{d_y}.
\]
Therefore, by using $\log(1+x)\leq x$ for $x \geq 0$, we conclude that 
\[
\begin{split}
& \log N(\mathcal{F}^{\sharp}_{t,j}, \epsilon, \|\cdot\|_{S} )\\
& \le 
\frac{10m^{\sharp}d_y\min\{\sqrt{m^{\sharp}}, \sqrt{d_y}\}R_{o}L_{o,t}\sqrt{n}}{\epsilon}+\frac{10m^{\sharp}d_x\min\{\sqrt{m^{\sharp}}, \sqrt{d_x}\}R_{i}L_{i,t}\sqrt{n}}{\epsilon}\\
& \qquad + 
\frac{10m^{\sharp}R^{b}_{hi}L_{b,t}\sqrt{n}}{\epsilon}+\frac{10(m^{\sharp})^{\frac{5}{2}}R_{h}L_{h,t}\sqrt{n}}{\epsilon}
+\frac{10m^{\sharp}d_{y}R^{b}_{o}\sqrt{n}}{\epsilon}\\
& =
\frac{10m^{\sharp}\sqrt{n}M_{t}}{\epsilon},
\end{split}
\]
where $M_{t}$ is the constant given by \eqref{def of M_t-Appendix}.
The proof of Lemma \ref{covering number bound} is finished.
\end{proof}

\section{Proof of Proposition \ref{bound of input infor}}\label{proof of Bach}
We review the following proposition (see Proposition 1 in [\cites{Suzuki} and Proposition 1 in [\cites{Bach}).
\begin{prop}\label{Bach}
Let $v_{1},\ldots,v_{m^{\sharp}}$ be i.i.d. sampled from the distribution $q$ in \eqref{distribution-q}, and $J=\{v_1,\ldots,v_{m^{\sharp}}\}$. Then, for any $\widetilde{\delta} \in (0, 1/2)$ and $\lambda>0$, if $m^{\sharp} \geq 5 \widehat{N}(\lambda) \log(16\widehat{N}(\lambda)/\widetilde{\delta})$, then we have the following inequality with probability greater than $1-\widetilde{\delta}$: 
\begin{equation}\label{Bach prop}
\inf_{\alpha \in \mathbb{R}^{m^{\sharp}}}\big\{ \big\|z^{T}\phi -\alpha^{T}\phi_{J} \big\|^{2}_{n,T}+\lambda m^{\sharp}\big\| \alpha^{T} \big\|^{2}_{\tau^{\prime}} \big\}\leq 4\lambda z^{T}\widehat{\Sigma}(\widehat{\Sigma}+\lambda I)^{-1}z,
\end{equation}
for all $z \in \mathbb{R}^{m}$.
\end{prop}
\begin{proof}
Let $e_j$ be an indicator vector which has $1$ at the $j$-th component and $0$ in other components for $j=1,\ldots,m$. 
Applying Proposition \ref{Bach} with $z=e_{j}$ and taking the summation over $j=1,\ldots,m$, we obtain
\[
\begin{split}
L^{(A)}_{\tau}(J)
&= 
\big\| \phi-\widehat{A}_{J}\phi_{J} \big\|^{2}_{n,T} +\lambda m^{\sharp} \big\|\widehat{A}_{J} \big\|^{2}_{\tau^{\prime}}\\
& \le \sum_{j=1}^{m}\big\{\big\| e_{j}^{T} \phi-e_{j}^{T}\widehat{A}_{J}\phi_{J} \big\|^{2}_{n,T}+\lambda m^{\sharp} \big\|e_{j}^{T}\widehat{A}_{J} \big\|^{2}_{\tau^{\prime}} \big\}\\
& = 
\sum_{j=1}^{m}\inf_{\alpha \in \mathbb{R}^{m^{\sharp}}}\big\{\big\| e_{j}^{T} \phi-\alpha^{T}\phi_{J} \big\|^{2}_{n,T}+\lambda m^{\sharp} \big\| \alpha^{T}\big\|^{2}_{\tau^{\prime}} \big\} \\
& \le 
4\lambda \sum_{j=1}^{m}e_{j}^{T}\widehat{\Sigma}(\widehat{\Sigma}+\lambda I)^{-1}e_{j}
\leq
4\lambda. 
\end{split}
\]
\end{proof}
\section{Proof of Theorem \ref{generalization error bound for SP}}\label{proof of generalization SP}

We restate Theorem \ref{generalization error bound for SP} in an exact form as follows:

\begin{thm}\label{thm:E}
Suppose that Assumptions \ref{assmption1}, \ref{assmption2} and \ref{assmption3} hold. 
Let $\{(X_{T}^{i}, Y_{T}^{i})\}_{i=1}^n$ and $\{v_j\}_{j=1}^{m^{\sharp}}$ be sampled i.i.d. from the distributions $P_T$ and $q$ in \eqref{distribution-q}, respectively. 
Let $J=\{v_1,\ldots,v_{m^{\sharp}} \}$. 
Then, for any $\delta \geq \log2$ and $\widetilde{\delta} \in (0,1/2)$, 
we have the following inequality with probability greater than $(1-2e^{-\delta})\widetilde{\delta}$:
\begin{equation}\label{general error bound-2-Appendix}
\begin{split}
\Psi_{j}(f^{\sharp}_{J}) 
& \leq \widehat{\Psi}_{j}(\widehat{f})\\
& +
\sqrt{3} \rho_{\psi} \Bigg\{
2\widehat{R}_{o} +
4\widehat{R}_{o}\rho_{\sigma}\sqrt{\frac{m}{1-2\widetilde{\delta}}}\max\bigg\{1, \bigg(2\rho_{\sigma}\widehat{R}_{h}\sqrt{\frac{m}{1-2\widetilde{\delta}}}\bigg)^{T-2}\bigg\}T\widehat{R}_{h}\Bigg\}\sqrt{\lambda}\\
& +
\frac{1}{\sqrt{n}}\Bigg\{\frac{\widehat{c}\rho_{\psi}\sqrt{m^{\sharp}}}{T}\bigg(\sum_{t=1}^{T}\widehat{M}_{t}^{1/2}\widehat{R}^{1/2}_{\infty, t}\bigg)  +3\sqrt{2 \delta}(\rho_{\psi}\widehat{R}_{\infty, T}+R_{y})\Bigg\}\\
& \lesssim 
\widehat{\Psi}_{j}(\widehat{f})+\sqrt{\lambda}+\frac{1}{\sqrt{n}}(m^{\sharp})^{\frac{5}{4}} \widehat{R}^{1/2}_{\infty, T},
\end{split}
\end{equation}
for $j=1,\ldots,d_y$ and for all $\lambda>0$ satisfying \eqref{Bach condition}, 
where $\widehat{R}_{\infty, t}$ and $\widehat{M}_{t}$ are defined by
\[
\widehat{R}_{\infty, t}:=2\rho_{\sigma}\widehat{R}_{o}\sqrt{\frac{m}{1-2\widetilde{\delta}}}(\widehat{R}_{i}R_{x}+\widehat{R}^{b}_{hi})
\Bigg\{\sum_{l=1}^{t}\bigg(2\rho_{\sigma}\widehat{R}_{h}\sqrt{\frac{m}{1-2\widetilde{\delta}}} \bigg)^{l-1} \Bigg\}+\widehat{R}^{b}_{o},
\]
\[
\begin{split}
\widehat{M}_{t}
& := 2\rho_{\sigma}\widehat{R}_{o}\sqrt{\frac{m}{1-2\widetilde{\delta}}}\Biggl\{ \Big( d_y\min\{\sqrt{m^{\sharp}}, \sqrt{d_y}\} + d_x\min\{\sqrt{m^{\sharp}}, \sqrt{d_x}\} \Big)\widehat{R}_{i}R_{x}\\
& +
\Big(d_y\min\{\sqrt{m^{\sharp}}, \sqrt{d_y}\} + 1\Big)\widehat{R}^{b}_{hi}\Biggr\}
\Bigg\{\sum_{l=0}^{t-1} \bigg(2\rho_{\sigma} \widehat{R}_{h}\sqrt{\frac{m}{1-2\widetilde{\delta}}}\bigg)^{l}\Bigg\}\\
& + 
4(m^{\sharp})^{3/2}\widehat{R}_{h}\widehat{R}_{o}\frac{m}{1-2\widetilde{\delta}}\rho_{\sigma}^{2}(\widehat{R}_{i}R_{x}+\widehat{R}^{b}_{hi})
\Bigg\{
\sum_{l=1}^{t-1}\sum_{k=0}^{l-1} \bigg(2\rho_{\sigma} \widehat{R}_{h}\sqrt{\frac{m}{1-2\widetilde{\delta}}}\bigg)^{t-1-l+k}\Bigg\}+d_{y}\widehat{R}^{b}_{o}.
\end{split}
\]
\end{thm}
\begin{proof} 
Let $\tilde{\delta} \in (0,1/2)$, and let 
$f^{\sharp}_{J}$ be the compressed RNN with parameters 
\[
W^{\sharp o}_{J}:=\widehat{W}^{o}\widehat{A}_{J},\quad W^{\sharp h}_{J}:=\widehat{W}^{h}_{J,[m]}\widehat{A}_{J},\quad W^{\sharp i}_{J}:=\widehat{W}^{i}_{J,[d_x]},\quad b^{\sharp hi}_{J}:=\widehat{b}^{hi}_{J},\quad \text{and} \quad b^{\sharp o}_{J}:=\widehat{b}^{o}. 
\]
Once we can prove that 
\begin{equation}\label{E-goal}
f^{\sharp}_{J} 
\in 
\mathcal{F}^{\sharp}_{T}\bigg(2\widehat{R}_{o}\sqrt{\frac{m}{1-2\widetilde{\delta}}}, 2\widehat{R}_{h}\sqrt{\frac{m}{1-2\widetilde{\delta}}}, \widehat{R}_i, \widehat{R}^{b}_{o},\widehat{R}^{b}_{hi}\bigg),
\end{equation}
we can apply Theorem \ref{thm:C} with $f^\sharp = f^\sharp_J$ to obtain, for any $\delta \geq \log2$, the following inequality with probability greater than $1-2e^{-\delta}$:
\begin{equation}\label{thmC-application}
\begin{split}
\Psi_{j}& (f^{\sharp}_J)
\le 
\widehat{\Psi}_{j}(\widehat{f})
+
\sqrt{3}\rho_{\psi} \Bigg\{
\big\| \widehat{W}^{o}\phi-W^{\sharp o}\phi_{J} \big\|_{n,T}\\
& + 2\widehat{R}_{o}\sqrt{\frac{m}{1-2\widetilde{\delta}}}\rho_{\sigma}\max\bigg\{1, \bigg(2\widehat{R}_{h}\sqrt{\frac{m}{1-2\widetilde{\delta}}} \rho_{\sigma}\bigg)^{T-2}\bigg\}T\big\| \widehat{W}^{h}_{J,[m]}\phi-W^{\sharp h}\phi_{J} \big\|_{n,T}\Bigg\}\\ 
& + 
\frac{1}{\sqrt{n}}\Bigg\{\frac{\widehat{c}\rho_{\psi}\sqrt{m^{\sharp}}}{T}\bigg(\sum_{t=1}^{T}\widehat{M}_{t}^{1/2}\widehat{R}^{1/2}_{\infty, t}\bigg)  +3\sqrt{2\delta} (\rho_{\psi}\widehat{R}_{\infty, T}+R_{y}) \Bigg\},
\end{split}
\end{equation}
for $j=1,\ldots,d_y$.
Moreover, by using Proposition \ref{Bach}, we have 
\begin{equation}\label{Bach application 1}
\begin{split}
\big\| \widehat{W}^{o} \phi-W^{\sharp o}_{J}\phi_{J} \big\|^{2}_{n,T}
& = 
\big\| \widehat{W}^{o} \phi-\widehat{W}^{o}\widehat{A}_{J}\phi_{J} \big\|^{2}_{n,T}\\
& \le \sum_{j=1}^{d_y}\Big( \big\| \widehat{W}^{o}_{j,:} \phi-\widehat{W}^{o}_{j,:}\widehat{A}_{J}\phi_{J} \big\|^{2}_{n,T}+\lambda m^{\sharp} \big\|\widehat{W}^{o}_{j,:}\widehat{A}_{J} \big\|^{2}_{\tau^{\prime}} \Big)\\
& =
\sum_{j=1}^{d_y}\inf_{\alpha \in \mathbb{R}^{m^{\sharp}}}\Big(\big\| \widehat{W}^{o}_{j,:} \phi-\alpha^{T}\phi_{J} \big\|^{2}_{n,T}+\lambda m^{\sharp} \big\| \alpha^{T}\big\|^{2}_{\tau^{\prime}} \Big) \\
& \le 
4\lambda \sum_{j=1}^{d_y}\widehat{W}^{o}_{j,:}\widehat{\Sigma}(\widehat{\Sigma}+\lambda I)^{-1}(\widehat{W}^{o}_{j,:})^{T}\\
& \leq
4\lambda \big\|\widehat{W}^{o} \big\|^{2}_{F}
\leq
4\lambda (\widehat{R}_{o})^2,
\end{split}
\end{equation}
and
\begin{equation}\label{Bach application 2}
\begin{split}
\big\| \widehat{W}^{h}_{J,[m]}\phi-W^{\sharp h}_{J}\phi_{J} \big\|^{2}_{n,T}
& = \big\| \widehat{W}^{h}_{J,[m]} \phi-\widehat{W}^{h}_{J,[m]}\widehat{A}_{J}\phi_{J} \big\|^{2}_{n,T}\\
& \le \sum_{j \in J}\Big(\big\| \widehat{W}^{h}_{j,:} \phi-\widehat{W}^{h}_{j,:}\widehat{A}_{J}\phi_{J} \big\|^{2}_{n,T}+\lambda m^{\sharp} \big\|\widehat{W}^{h}_{j,:}\widehat{A}_{J} \big\|^{2}_{\tau^{\prime}} \Big)\\
& = 
\sum_{j \in J}\inf_{\alpha \in \mathbb{R}^{m^{\sharp}}}\Big(\big\| \widehat{W}^{h}_{j,:} \phi-\alpha^{T}\phi_{J} \big\|^{2}_{n,T}+\lambda m^{\sharp} \big\| \alpha^{T}\big\|^{2}_{\tau^{\prime}} \Big) \\
& \le 
4\lambda \sum_{j \in J}\widehat{W}^{h}_{j,:}\widehat{\Sigma}(\widehat{\Sigma}+\lambda I)^{-1}(\widehat{W}^{h}_{j,:})^{T}\\
& \leq
4\lambda \big\|\widehat{W}^{h} \big\|^{2}_{F}
\leq
4\lambda (\widehat{R}_{h})^2.
\end{split}
\end{equation}
Therefore, by combining \eqref{thmC-application}, \eqref{Bach application 1} and \eqref{Bach application 2}, 
we conclude the inequality \eqref{general error bound-2-Appendix}. 
It remains to prove \eqref{E-goal}. 
Finally, 
we prove that \eqref{E-goal} holds with probability greater than $\widetilde{\delta}$. 
\par
Let us recall the definition \eqref{leverage} of the leverage score   
$\tau^{\prime}=(\tau^{\prime}_{j})_{j \in J}\in \mathbb{R}^{m^{\sharp}}$, i.e., 
\[
\tau^{\prime}_{j}:=\frac{1}{\widehat{N}(\lambda)}\big[ \widehat{\Sigma}( \widehat{\Sigma}+\lambda I )^{-1} \big]_{j,j},\quad j=1,\cdots,m.
\]
By Markov's inequality, we have
\begin{equation}\label{Markov ineq}
\begin{split}
P\bigg[\sum_{j \in J}(\tau^{\prime}_{j})^{-1}<\frac{mm^{\sharp}}{1-2\widetilde{\delta}} \bigg]
& = 
1 - P\bigg[\sum_{j \in J}(\tau^{\prime}_{j})^{-1}\geq\frac{mm^{\sharp}}{1-2\widetilde{\delta}} \bigg]\\
& \ge 
1-\frac{E\big[\sum_{j \in J}(\tau^{\prime}_{j})^{-1}\big]}{\frac{mm^{\sharp}}{1-2\widetilde{\delta}}}=2\widetilde{\delta}, 
\end{split}
\end{equation}
because $E\big[\sum_{j \in J}(\tau^{\prime}_{j})^{-1}\big]=mm^{\sharp}$ (see the proof of Lemma 1 in [\cites{Suzuki}). 
Therefore, the probability of two events (\ref{Bach prop}) and 
\begin{equation}
\sum_{j \in J}(\tau^{\prime}_{j})^{-1}<\frac{mm^{\sharp}}{1-2\widetilde{\delta}}, \label{constrain}
\end{equation}
happening simultaneously is greater than $(1-\widetilde{\delta})+2\widetilde{\delta}-1=\widetilde{\delta}$.
By the same argument as in \eqref{Bach application 1} and \eqref{Bach application 2}, and by using \eqref{Markov ineq}, 
we have 
\[
\begin{split}
\big\|W^{\sharp o}_{J} \big\|^{2}_{F}
& = \frac{\lambda m^{\sharp}}{\lambda m^{\sharp}}\big\| \widehat{W}^{o}\widehat{A}_{J}\big\|^{2}_{F}\\
& \le 
\frac{(\sum_{j \in J}(\tau^{\prime}_{j})^{-1})}{\lambda m^{\sharp}}
\sum_{j=1}^{d_y}\Big(\big\| \widehat{W}^{o}_{j,:} \phi-\widehat{W}^{o}_{j,:}\widehat{A}_{J}\phi_{J} \big\|^{2}_{n,T}+\lambda m^{\sharp} \big\|\widehat{W}^{o}_{j,:}\widehat{A}_{J} \big\|^{2}_{\tau^{\prime}} \Big)\\
& \le 
4(\widehat{R}_{o})^{2}\frac{m}{1-2\widetilde{\delta}},
\end{split}
\]
\[
\begin{split}
\big\|W^{\sharp h}_{J} \big\|^{2}_{F}
& = \frac{\lambda m^{\sharp}}{\lambda m^{\sharp}}\big\| \widehat{W}^{h}_{J,[m]}\widehat{A}_{J}\big\|^{2}_{F}\\
& \le 
\frac{(\sum_{j \in J}(\tau^{\prime}_{j})^{-1})}{\lambda m^{\sharp}}
\sum_{j \in J}\Big(\big\| \widehat{W}^{h}_{j,:} \phi-\widehat{W}^{h}_{j,:}\widehat{A}_{J}\phi_{J} \big\|^{2}_{n,T}+\lambda m^{\sharp} \big\|\widehat{W}^{h}_{j,:}\widehat{A}_{J} \big\|^{2}_{\tau^{\prime}} \Big)\\
& \le 
4(\widehat{R}_{h})^{2}\frac{m}{1-2\widetilde{\delta}},
\end{split}
\]
and 
\[
\big\|W^{\sharp i}_{J} \big\|^{2}_{F} \leq \big\|\widehat{W}^{i} \big\|^{2}_{F}\leq (\widehat{R}_{i})^{2},
\quad 
\big\|b^{\sharp o}_{J} \big\|^{2}_{F} \leq \big\|\widehat{b}^{o} \big\|^{2}_{F}\leq (\widehat{R}_{o}^{b})^{2},
\quad 
\big\|b^{\sharp hi}_{J} \big\|^{2}_{F} \leq \big\|\widehat{b}^{hi} \big\|^{2}_{F}\leq (\widehat{R}_{hi}^{b})^{2}.
\]
Hence, \eqref{E-goal} holds with probability greater than $\widetilde{\delta}$. 
Thus, we conclude Theorem \ref{thm:E}. 
\end{proof}
\section{Remarks for Theorems \ref{generalization error bound for SP} and \ref{thm:E}}\label{Remarks for}
\begin{rem}
We remark that the index $J$ in Theorem \ref{generalization error bound for SP} is a random variable with a distribution $q$. If the deterministic $J$ satisfying (\ref{Bach prop}) and (\ref{constrain}) is considered, the inequality (\ref{general error bound-2}) holds with a probability greater than $1-2e^{-\delta}$, which is the same probability obtained with the inequality in Theorem 2 of [\cites{Suzuki}. The index $J$ in Theorem 2 of [\cites{Suzuki} is chosen deterministically by minimizing the information losses (2) with the additional constraint $\sum_{j\in J}(\tau^{\prime}_{j})^{-1}<\frac{5}{3}mm^{\sharp}$. This constraint can be interpreted as the leverage score $\tau^{\prime}_{J}$ corresponding to $J$ becomes larger, which implies that important nodes are selected from the spectral information of the covariance matrix $\widehat{\Sigma}$.
\end{rem}
\begin{rem}
In the case of $m>m_{\mathrm{nzr}}$, we can obtain a sharper error bound than (\ref{general error bound-2}) in Theorem~\ref{generalization error bound for SP}. More precisely, the constant omitted in (\ref{general error bound-2}), which depends on the size $m$ of $\widehat{f}$, can be improved to 
the constant depending on $m_{\mathrm{nzr}}$, not on $m$. In fact, when $m>m_{\mathrm{nzr}}$, let $\widehat{f}_{\mathrm{nzr}}$ be the network obtained by deleting the nodes corresponding to the non-zero rows of the covariance matrix $\widehat{\Sigma}$. By the same argument, replacing $\widehat{\Psi}_{j}(\widehat{f})$ with $\widehat{\Psi}_{j}(\widehat{f}_{\mathrm{nzr}})$ in the proof of Theorem \ref{generalization error bound for SP}, we can obtain Theorem \ref{generalization error bound for SP} by replacing $m$ by $m_{\mathrm{nzr}}$, which means that a sharper error bound can be obtained.

\end{rem}
\section{Detailed configurations for training, pruning and fine-tuning}
\label{Detailed configurations for training, pruning and fine-tuning}

Employed architecture for the Pixel-MNIST classification task consists of a single IRNN layer and an output layer,
while that for the PTB word level language modeling consists of an embedding layer, a single RNN layer and an output layer, where we can merge an embedding weight matrix and an RNN input weight matrix into an single weight matrix.
The loss function is the cross entropy function following the soft-max function for both tasks.
Each training and fine-tuning is optimized by Adam, and hyper-parameters obtained by grid search are summarized in Table \ref{hyperparameters}, where ``FT'' means the parameter used in fine-tuning and ``bptt'' means the step size for back-propagation through time.
As regards regularization techniques for the PTB task, we adopt the dropout, whose ratio is $0.1$, in any case and the weight tying [\cites{inan2016tying} in effective case.

\begin{table}[h]
\caption{Hyper-parameters for learning.}
\label{hyperparameters}
\centering
{\small
\begin{tabular}{@{}lcccccc@{}}
\toprule
Task  & epochs (FT) & batch size & learning rate (FT) & LR decay (step) & gradient clip & bptt \\
\midrule
Pixel-MNIST & 500 (250) & 120 & $10^{-4}$ ($5^{-5}$) & 0.95 (10) & 1.0 & 784 \\
PTB         & 200 (200) &  20 & $5.0$ \ ($2.5$) & 0.95 (1) & 0.01 & 35 \\
\bottomrule
\end{tabular}
}
\end{table}

We sample five models for each baseline in section \ref{Numerical Experiments}.
Furthermore, pruning methods including randomness are applied five times for each baseline model.
Other detailed configurations for each method are the following:
\begin{itemize}
\item Baseline (128)
    \begin{itemize}
    \item train:
        \begin{itemize}
        \item hidden size: 128
        \item weight tying: True
        \end{itemize}
    \end{itemize}
\item Baseline (42)
    \begin{itemize}
    \item train:
        \begin{itemize}
        \item hidden size: 42
        \item weight tying: True
        \end{itemize}
    \item prune:
        \begin{itemize}
        \item None
        \end{itemize}
    \item finetune: (only PTB case)
        \begin{itemize}
        \item hidden size: 42 (stay)
        \item weight tying: False
        \end{itemize}
    \end{itemize}
\item Spectral w/ rec. or w/o rec.
    \begin{itemize}
    \item train:
        \begin{itemize}
        \item Use Baseline (128)
        \end{itemize}
    \item prune:
        \begin{itemize}
        \item size of hidden-to-hidden weight matrix: $16384 (=128 \times 128) \to 1764 (=42 \times 42)$
        \item size of input-to-hidden weight matrix: $128 (=1 \times 128) \to 42(=1 \times 42)$ (Pixel-MNIST) or $1270016(=9922 \times 128) \to 416724(=9922 \times 42)$ (PTB)
        \item size of hidden-to-output weight matrix: $1280 (=128 \times 10) \to 420 (=42 \times 10)$ (Pixel-MNIST) or $1270016(=9922 \times 128) \to 416724(=9922 \times 42)$ (PTB)
        \item Reduce the RNN weight matrices based on our proposed method with or without the reconstruction matrix
        \end{itemize}
    \item finetune:
        \begin{itemize}
        \item hidden size: 42 (reduced from 128)
        \item weight tying: False
        \end{itemize}
    \end{itemize}
\item Random w/ rec. or w/o rec.
    \begin{itemize}
    \item Same as ``Spectral'' except for reducing the RNN weight matrices randomly in pruning phase
    \end{itemize}
\item Column Sparsification
    \begin{itemize}
    \item train:
        \begin{itemize}
        \item hidden size: 128
        \item weight tying: True
        \item Mask the lowest $86(=128 - 42)$ columns of the hidden-to-hidden weight matrix by $L^2$-norm for each iteration (add noise on the weight matrix before masking when applied to the IRNN)
        \end{itemize}
    \item prune:
        \begin{itemize}
        \item Fix the mask
        \end{itemize}
    \item finetune:
        \begin{itemize}
        \item None
        \end{itemize}
    \end{itemize}
\item Low Rank Factorization
    \begin{itemize}
    \item train:
        \begin{itemize}
        \item Use Baseline (128)
        \end{itemize}
    \item prune:
        \begin{itemize}
        \item intrinsic parameters of hidden-to-hidden weight matrix: $16384(=128 \times 128) \to 10752(=128 \times 42 + 42 \times 128)$
        \item Decompose hidden-to-hidden weight matrix based on SVD: $W = U S V^\top \to W' = U[:, :42]S[:42] V^\top[:42, :]$
        \item Entry of $S$, which is singular values, are in descending order
        \end{itemize}
    \item finetune:
        \begin{itemize}
        \item None
        \end{itemize}
    \end{itemize}
\item Magnitude-based Weight
    \begin{itemize}
    \item train:
        \begin{itemize}
        \item Use Baseline (128)
        \end{itemize}
    \item prune:
        \begin{itemize}
        \item parameters of hidden-to-hidden weight matrix: $16384(=128 \times 128) \to 1764(=42 \times 42)$
        \item Remove the lowest $14620(=128 \times 128 - 42 \times 42)$ parameters by $L^1$-norm
        \end{itemize}
    \item finetune:
        \begin{itemize}
        \item hidden size: 128 (stay but have sparse weight matrix)
        \end{itemize}
    \end{itemize}
\item Random Weight
    \begin{itemize}
    \item Same as ``Magnitude-based Weight'' except for removing parameters randomly in pruning phase
    \end{itemize}
\end{itemize}


\end{document}